\documentclass[10pt,journal,compsoc]{IEEEtran}
\ifCLASSOPTIONcompsoc
  \usepackage[nocompress]{cite}
\else
  \usepackage{cite}
\fi
\ifCLASSINFOpdf
\else
\fi

\usepackage{times}
\usepackage{epsfig}
\usepackage{graphicx}
\usepackage{amsmath}
\usepackage{amssymb}
\usepackage{color}
\usepackage{xcolor}
\usepackage{hyperref}
\usepackage{booktabs}
\usepackage{multirow}
\usepackage{xspace}
\usepackage{algorithm}
\usepackage{algorithmic}

\newcommand{\etal}{\emph{et al.}\xspace}

\newcommand{\ie}{\emph{i.e.}\xspace}

\newcommand{\figLabel}{Figure\xspace}
\newcommand{\eqLabel}{Equation\xspace}
\newcommand{\secLabel}{Section\xspace}
\newcommand{\tblLabel}{Table\xspace}
\newcommand{\mysection}[1]{\vspace{3pt}\noindent\textbf{#1.}}

\definecolor{orange}{rgb}{1,0.5,0}
\definecolor{maroon}{rgb}{0.51,0,0}

\usepackage{amsthm}
\newtheorem{proposition}{Proposition}
\newtheorem{theorem}{Theorem}
\begin{document}
%
% paper title
% Titles are generally capitalized except for words such as a, an, and, as,
% at, but, by, for, in, nor, of, on, or, the, to and up, which are usually
% not capitalized unless they are the first or last word of the title.
% Linebreaks \\ can be used within to get better formatting as desired.
% Do not put math or special symbols in the title.
\title{GPNet: Simplifying Graph Neural Networks via Multi-channel Geometric Polynomials}
%\\\small\url{ https://sites.google.com/view/deep-gcns}}
%
%
% author names and IEEE memberships
% note positions of commas and nonbreaking spaces ( ~ ) LaTeX will not break
% a structure at a ~ so this keeps an author's name from being broken across
% two lines.
% use \thanks{} to gain access to the first footnote area
% a separate \thanks must be used for each paragraph as LaTeX2e's \thanks
% was not built to handle multiple paragraphs
%
%
%\IEEEcompsocitemizethanks is a special \thanks that produces the bulleted
% lists the Computer Society journals use for "first footnote" author
% affiliations. Use \IEEEcompsocthanksitem which works much like \item
% for each affiliation group. When not in compsoc mode,
% \IEEEcompsocitemizethanks becomes like \thanks and
% \IEEEcompsocthanksitem becomes a line break with idention. This
% facilitates dual compilation, although admittedly the differences in the
% desired content of \author between the different types of papers makes a
% one-size-fits-all approach a daunting prospect. For instance, compsoc
% journal papers have the author affiliations above the "Manuscript
% received ..."  text while in non-compsoc journals this is reversed. Sigh.

\author{Xun~Liu$^{1,2}$, Alex Hay-Man~Ng$^{1,3}${$^{\ast}$}, Fangyuan~Lei$^{4,5}${$^{\ast}$}, \emph{Member, IEEE}, Yikuan~Zhang$^2$, Zhengmin~Li$^{5}$\\%etc.\\hayman.ng@gdut.edu.cn
\thanks{*Equal corresponding author}
$^1$School of Information Engineering, Guangdong University of Technology\\
$^2$Department of Electronics, Software Engineering Institute of Guangzhou\\
$^3$School of Civil and Transportation Engineering, Guangdong University of Technology\\
$^4$Guangdong Provincial Key Laboratory of Intellectual Property and Big Data, Guangdong Polytechnic Normal University\\
$^5$School of Cyber Security, Guangdong Polytechnic Normal University\\
\texttt{{liuxun.stf}@gmail.com}
\quad \texttt{{hayman.ng}@gdut.edu.cn}
\quad \texttt{{leify}@gpnu.edu.cn}
\quad \texttt{{zyk}@mail.seig.edu.cn}
\quad \texttt{{lizhengming2004}@126.com}
}

\IEEEtitleabstractindextext{%
%%%%%%%%% ABSTRACT
\begin{abstract}
Graph Neural Networks (GNNs) are a promising deep learning approach for circumventing many real-world problems on graph-structured data. However, these models usually have at least one of four fundamental limitations: over-smoothing, over-fitting, difficult to train, and strong homophily assumption. For example, Simple Graph Convolution (SGC) is known to suffer from the first and fourth limitations. To tackle these limitations, we identify a set of key designs including (D1) dilated convolution, (D2) multi-channel learning, (D3) self-attention score, and (D4) sign factor to boost learning from different types (\ie homophily and heterophily) and scales (\ie small, medium, and large) of networks, and combine them into a graph neural network, GPNet, a simple and efficient one-layer model. We theoretically analyze the model and show that it can approximate various graph filters by adjusting the self-attention score and sign factor. Experiments show that GPNet consistently outperforms baselines in terms of average rank, average accuracy, complexity, and parameters on semi-supervised and full-supervised tasks, and achieves competitive performance compared to state-of-the-art model with inductive learning task.
\end{abstract}

% Note that keywords are not normally used for peerreview papers.
}

% make the title area
\maketitle

% To allow for easy dual compilation without having to reenter the
% abstract/keywords data, the \IEEEtitleabstractindextext text will
% not be used in maketitle, but will appear (i.e., to be "transported")
% here as \IEEEdisplaynontitleabstractindextext when the compsoc
% or transmag modes are not selected <OR> if conference mode is selected
% - because all conference papers position the abstract like regular
% papers do.
%\IEEEdisplaynontitleabstractindextext
% \IEEEdisplaynontitleabstractindextext has no effect when using
% compsoc or transmag under a non-conference mode.

% For peer review papers, you can put extra information on the cover
% page as needed:
% \ifCLASSOPTIONpeerreview
% \begin{center} \bfseries EDICS Category: 3-BBND \end{center}
% \fi
%
% For peerreview papers, this IEEEtran command inserts a page break and
% creates the second title. It will be ignored for other modes.
\IEEEpeerreviewmaketitle

%\MM{When a TPAMI submission is based on a previous conference paper, IEEE requires that the journal paper be a ¡°substantial revision¡± of the previous publication (30 percent is generally considered ¡°substantial¡±). TPAMI interprets and applies this requirement on a case-by-case basis with appropriate deference to the author¡¯s viewpoint. Examples of the improvements we expect to see over the conference paper include the following: additional technical details, a clearer explanation of the contribution, more experiments if appropriate, or an updated state-of-the-art. Of course, the authors are also encouraged to make the journal version a significant improvement on the conference paper (for example, by taking the opportunity to integrate their previous work or performing additional substantive work to answer questions that their conference paper raised). Since the journal version is intended to be the definitive, archival version of the research, TPAMI expects that the authors will take this opportunity to further improve their conference paper.}

%%%%%%%%% BODY TEXT
%\vspace{-8pt}
\IEEEraisesectionheading{\section{Introduction}\label{sec:introduction}}
% Computer Society journal (but not conference!) papers do something unusual
% with the very first section heading (almost always called "Introduction").
% They place it ABOVE the main text! IEEEtran.cls does not automatically do
% this for you, but you can achieve this effect with the provided
% \IEEEraisesectionheading{} command. Note the need to keep any \label that
% is to refer to the section immediately after \section in the above as
% \IEEEraisesectionheading puts \section within a raised box.

% The very first letter is a 2 line initial drop letter followed
% by the rest of the first word in caps (small caps for compsoc).
%
% form to use if the first word consists of a single letter:
% \IEEEPARstart{A}{demo} file is ....
%
% form to use if you need the single drop letter followed by
% normal text (unknown if ever used by the IEEE):
% \IEEEPARstart{A}{}demo file is ....
%
% Some journals put the first two words in caps:
% \IEEEPARstart{T}{his demo} file is ....
%
% Here we have the typical use of a "T" for an initial drop letter
% and "HIS" in caps to complete the first word.

%\hfill October 1, 2019
\IEEEPARstart{M}{otivated} by the excellent performance of Deep Learning (DL) in a wide range of applications, researchers have attempted to generalize DL models to graph-structured data. Their methods are collectively referred to as Graph Neural Networks (GNNs).  Recently, GNNs have shown natural advantages on various graph learning tasks, including citation networks \cite{1,2,3}, traffic forecasting \cite{4,5,6}, recommendation system \cite{7,8}, social analysis \cite{9,10}, biology \cite{11,12}, drug discovery \cite{13,14}, and computer vision \cite{15,16}. However, these models have at least one of several fundamental limitations are summarized in the following four-fold.

\textbf{Limitation 1} \emph{Over-smoothing.} The convolution in many GNNs such as Graph Convolution Network (GCN) \cite{1} is essentially a special form of Laplacian smoothing \cite{17}. The smoothing operation obtains the smoothing node representations of a node and its close neighbors, thus making the features of nodes in the same cluster similar. By applying the smoothing operation repeatedly (stacking many layers, deep GNNs), the features of nodes in same connected component converge to the fixed value and are difficult to distinguish \cite{17,18}. This is the well-known problem of over-smoothing in GNNs. Residual connection \cite{19} is an effective technology of training very deep neural networks. However, applying the residual connection in GNNs merely slows down the over-smoothing \cite{1}. The classification performance of these GNNs drops as the number of layers increases (exceed two layers).

\textbf{Limitation 2} \emph{Over-fitting.} Several methods try to address the over-smoothing of deep GNNs with their redefined convolutions. Xu \etal \cite{20} use dense jump connections to flexibly leverage the features of various neighbors for each node. DropEdge \cite{21} randomly deletes a few edges from the input graph to relieve the concerns of over-smoothing. GCNII \cite{22} suggests that GNNs can be generalized to a deep model with the two techniques of initial residual and identity mapping. While these models widen the receptive field of GNNs by increasing the number of network layers, training such models on small graphs may suffer from over-fitting due to the increased substantial parameters to construct these deep GNNs.

\textbf{Limitation 3} \emph{Difficult to train.} To avoid the over-fitting, more training data and computational resources are required while keeping these models unchanged, which makes it difficult to train on large-scale graphs.

\textbf{Limitation 4} \emph{Strong homophily assumption.} The convolution of most GNNs such as GCN and Simple Graph Convolution (SGC) \cite{23} can be viewed as a special form of non-negative low-pass filter, which is the key to the success of these models for homophily datasets (similar nodes tend to connect with each other). However, the filter filters out the high-frequency signals of node features, thus being distant from optimal with heterophily datasets (Connected nodes usually belong to different classes). In other word, the low-frequency and high-frequency signals of the original features may be optimal for homophily and heterophily graphs, respectively. These observations are highlighted by \cite{24,25}.

A recent work tends to simplify GCN. Wu \etal \cite{23} propose SGC, an extremely efficient one-layer model, with competitive performance to GCN in various graph learning tasks, by repeatedly removing the nonlinearities between GCN layers and precomputing the fixed feature extraction. This implies that most classifiers in a wide range of applications are linear. Benefit from the design, SGC has an optimal speed and almost no the concerns of over-fitting, due to the single-layer framework and fixed feature extraction. Nevertheless, SGC suffers from the over-smoothing when capturing the features of neighboring nodes at long distances \cite{26}. In addition, SGC assumes strong homophily and may perform worse with heterophily datasets. Therefore, we require a model to tackle these limitations of SGC. If these limitations can be solved while retaining the advantages, the four fundamental problems of current GNN models are addressed naturally.

In this work, we propose Graph Neural Network via Multi-channel Geometric Polynomials (GPNet), a novel simple and efficient one-layer model that aims to learn to fit a distribution under optimal calculations and parameters on homophily and heterophily datasets (\figLabel \ref{fig:1}). From the one-layer design, GPNet has no concerns of overfitting and difficulty training on graphs (address limitations 2 and 3). GPNet can be viewed as a graph analogue of dilated convolution \cite{27}, which allows for enlarging the receptive field without loss of resolution by applying a dilation factor. In GPNet, we propose to combine the multi-channel geometric polynomials of the normalized adjacency matrices (with or without self-loops) of various dilation factors to increase the receptive field without bringing over-smoothing (address limitation 1). The main challenge is how to use signals of different frequencies to adapt to different types of networks. To tackle the challenge, we design the two simple techniques of self-attention score and sign factor to achieve performance improvement for both heterophily and homophily (address limitation 4). Theoretical analysis shows that GPNet can fit different types of filters via the adjustment of the self-attention score and sign factor to adapt to both homophily and heterophily networks. We summarize the main contributions as follows:

\textbf{Key Designs in Both Homophily and Heterophily.} We provide several key designs to generalize GNNs to heterophily settings without trading off accuracy in homophily: (D1) dilated convolution, (D2) multi-channel learning, (D3) self-attention score, and (D4) sign factor.

\textbf{GPNet Framework.} We propose GPNet, a simple and efficient one-layer model that learns graph representations for different types of networks in an end-to-end fashion, to tackle the four fundamental limitations of GNNs.

\textbf{Comprehensive Evaluation.} Extensive experiments on ten real-world networks demonstrate that GPNet achieves better accuracy with similar complexity and parameters over SGC, and compares favorably against other baselines in terms of classification accuracy, parameters, and complexity on both homophily and heterophily graphs.

The rest of the paper is organized as follows. In \secLabel \ref{sec:related}, we briefly introduce relevant GNN literature. We propose our key designs and GPNet architecture in \secLabel \ref{sec:methodology}. \secLabel \ref{sec:analysis} presents the analysis of SGC, the proposed GPNet, and relation between GPNet and representative models. In \secLabel \ref{sec:experiments}, we evaluate the performance of GPNet on various graph learning tasks. \secLabel \ref{sec:conclusion} summarizes the paper.

%\hfill mds

% \subsection{Subsection Heading Here}
% Subsection text here.

% % needed in second column of first page if using \IEEEpubid
% %\IEEEpubidadjcol

% \subsubsection{Subsubsection Heading Here}
% Subsubsection text here.

\section{Related Work}
\label{sec:related}
Motivated by the limited learning ability of Convolution Neural Networks (CNNs) for dealing with non-Euclidean applications, Spectral CNN \cite{28} and ChebyShev \cite{29} use graph signal theory to first extend CNNs on graphs. Kipf and Welling \cite{1} propose GCN by simplifying the Chebyshev polynomials of \cite{29}. Li \etal \cite{17} point out that each convolution of GCN is essentially equivalent to an operation of Laplacian smoothing, which brings the concerns of over-smoothing when repeatedly applying the operation. Many GNNs are generally shallow \cite{30} due to the over-smoothing. This shallow naturally limits the learning ability of obtaining the features of neighboring nodes at different distances. Several works design deep GNNs to improve the learning ability of graph representation. Li \etal \cite{31} apply the concepts of CNNs to construct a 56-layer GNN. Xu \etal \cite{20} propose three models of jumping knowledge (JK) networks based on three aggregation schemes of concatenation, max-pooling, and LSTM-attention to adjust the range of information aggregated by each node according to different positions and structures of the graph. Rong \etal \cite{21} propose the DropEdge architecture to relieve the concerns of over-fitting and over-smoothing in deep GNNs. By designing the two effective techniques of identity mapping and initial residual, Chen \etal \cite{22} propose Graph Convolutional Network via Initial Residual and Identity Mapping (GCNII) to achieve better performance compared to baselines on various semi-supervised and full-supervised tasks. Several researches devote to leverage the information based on sampling to improve the scalability of GNNs \cite{9,32,34,35}. Graph attentional models suggest adjusting edge weights at each layer to improve the expressive power of learning graph representations on various graph learning tasks \cite{2,37,38,39,40}. Several works directly capture the multi-hop neighborhood information of nodes through higher-order graph convolution \cite{3,41,42,43,44,add_2} instead of traditional convolutions with stacked layers.

Recently, researchers have focused on simpler and efficient models to significantly reduce training time in downstream applications. Thekumparampil \etal~\cite{37} design a linear GNN model that removes all intermediate non-linear activation functions to simplify computations. SGC~\cite{23} achieves comparable performance under optimal speed and parameters to GCN by removing the nonlinearities between GCN layers and collapsing the weight matrices between consecutive layers. gfNN~\cite{add_1} addresses the limitation that SGC and GCN may work worse in two cases of noisy features and nonlinear features spaces, while retaining the speed of SGC.

Nevertheless, most of the GNN models mentioned above are designed only for homophily datasets, which may not work well on heterophily datasets. Geom-GCN \cite{45}, ${\rm H_{2}GCN}$ \cite{24}, FAGCN \cite{46}, TDGNN \cite{47}, GPRGNN \cite{48}, and CPGNN \cite{25} that are explicitly designed to perform well on heterophily graphs. However, these models need to be improved in terms of efficiency, test accuracy, and parameters, and may fail to work on large-scale graphs.

\section{Methodology}
\label{sec:methodology}

\subsection{Preliminaries} \label{sec:Pre}

\mysection{Notations} We follow \cite{1} to describe an undirected graph $\mathcal{G}=(\mathcal{V}, \mathcal{E}, X, A)$ with $e$ edges and $n$ nodes (vertices), where $\mathcal{E}$ and
$\mathcal{V}$ denote the edge set and node set respectively,  $X \in R^{n\times d}$ is the initial feature matrix for all nodes of $\mathcal{G}$ assuming each node has $d$ features, and $A \in \{0, 1\}^{n\times n}$ is the adjacency matrix with each entry $a_{km}=1$ if there exists an edge (directly connected) between vertex $v_{k}$ and $v_{m}$, otherwise $a_{km}=0$. 
\begin{figure*}[!htb]
	\centering
	\includegraphics[width=\textwidth]{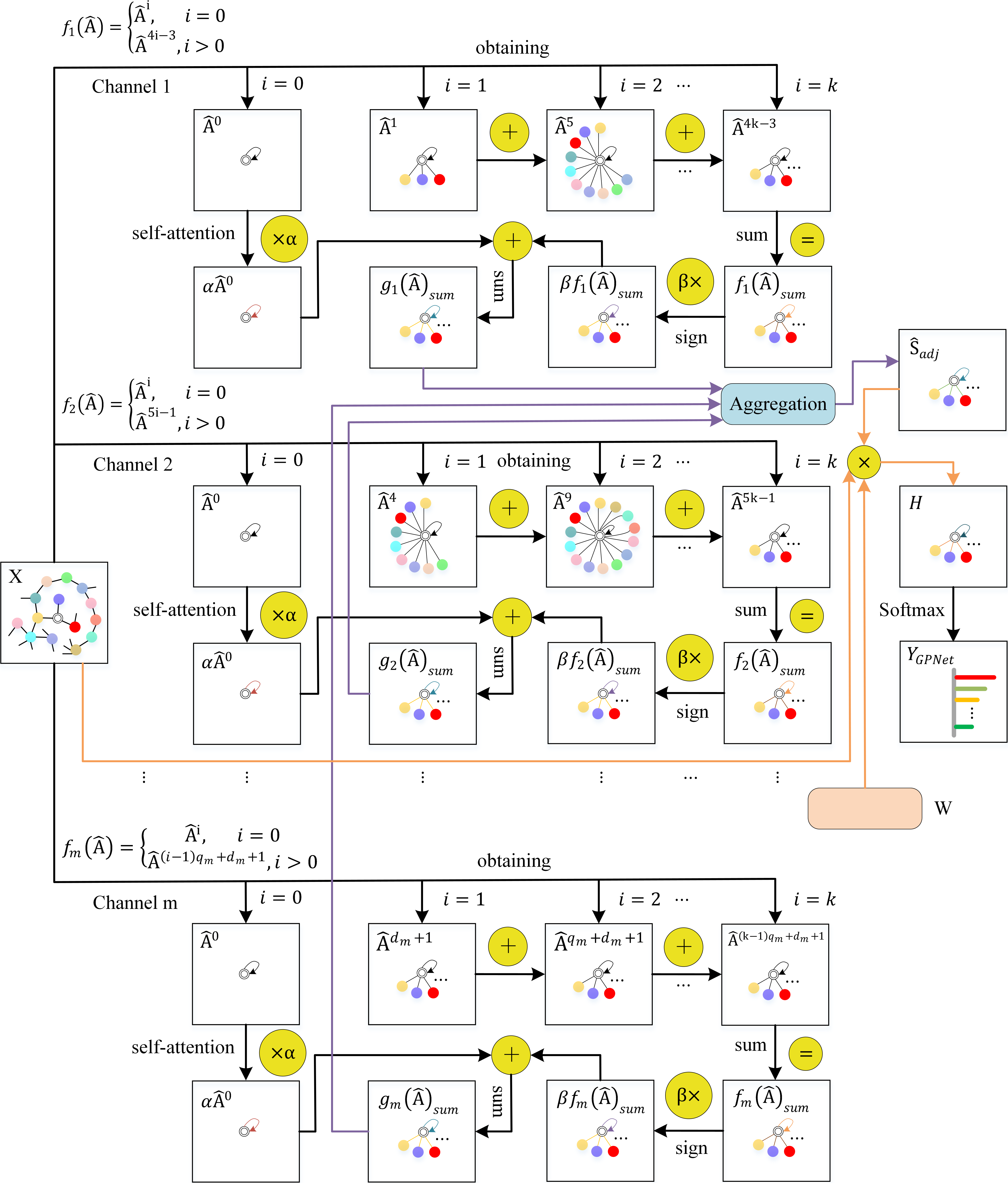}
	\caption{\textbf{The overall framework of GPNet} ($q_{0}=1$, $q_{1}=4$, $q_{2}=5$, $d_{1}=0$, $d_{2}=3$, self-loops).
	}
	\label{fig:1}
\end{figure*}
We can utilize an identity matrix $I$ to add the self-loops of $\mathcal{G}$. Most of GNN models leverage an adjacency matrix with self-loops $\tilde{A}$ to describe the edge relationships of $\mathcal{G}$, with $\tilde{A}=A+I$. The graph Laplacian matrix is defined as $L=D-A$, where $D$ is the degree matrix of $\mathcal{G}$. Its normalized version $L_{sym}=D^{-\frac{1}{2}}LD^{-\frac{1}{2}}$ is a symmetric positive semidefinite matrix with eigendecomposition $L=U\Lambda U^{T}$, where $U$ is orthonormal eigenvectors and $L=\Lambda=diag(\lambda_{1},\dots, \lambda_{n})$ denotes a diagonal matrix of eigenvalues $\lambda_{1},\dots, \lambda_{n}$.

\mysection{Graph Convolution Network} As described in GCN \cite{1}, each convolution in GCN updates node representations by capturing information from neighbors, which includes three stages: feature propagation, linear transformation, and nonlinear activation.

\textbf{Feature propagation} is the unique mechanism of aggregating information from neighbors in GCN compared to Multi-layer Perceptron (MLP). We can write the mechanism in sparse matrix form as:
\begin{equation} \label{eq1}
\hat{H}^{(l)}\gets S_{sym}H^{(l)} \,,
\end{equation}
where $H^{(l)}$ denotes the input node representations at layer $l$ with $H^{(l)}=X$ if $l=1$. We use $S_{sym}$ (a symmetric normalized adjacency matrix with added self-loops, with $S_{sym}=\tilde{D}^{-\frac{1}{2}}\tilde{A}\tilde{D}^{-\frac{1}{2}}$) to smooth the node representations locally along the edges between nodes. Here $\tilde{D}$ is the degree matrix of $\tilde{A}$.

\textbf{Linear transformation and nonlinear activation.} After the feature propagation, the smoothed representations are transformed linearly via a learned weight matrix $\theta^{(l)}$. Finally, a nonlinear activation function ReLU is widely used in GNNs to achieve the pointwise nonlinear transformation of the $l$-th layer.
\begin{equation} \label{eq2}
 H^{(l+1)}\gets \text{ReLU}(\hat{H}^{(l)}\theta^{(l)}) \,.
\end{equation}
Based on Equations \eqref{eq1}-\eqref{eq2}, we can write the following Equation for the convolution of GCN.
\begin{equation} \label{eq3}
 H^{(l+1)}\gets \text{ReLU}(S_{sym}H^{(l)}\theta^{(l)}) \,.
\end{equation}
By stacking the convolution twice and using a softmax classifier in the last layer to predict labels, we can write a classic 2-layer GCN:
\begin{equation} \label{eq4}
 Y_{\text{GCN}}=\text{softmax}(S_{sym}\text{ReLU}(S_{sym}X\theta^{(1)})\theta^{(2)}) \,,
\end{equation}
where $\theta^{(1)}$ and $\theta^{(2)}$ are different weight matrices.

\mysection{Simple graph convolution} By applying GCN layers repeatedly, we scale GCN to a $k$-layer architecture.
\begin{equation} \label{eq5}
    \hat{Y}=\text{softmax}(S_{sym}(\cdots \text{ReLU}(S_{sym}X\theta^{(1)})\cdots)\theta^{(k)}) \,.
\end{equation}
Suppose the nonlinearity between GCN layers is not critical to performance improvement. Therefore, for $k$-layer GCN, we delete all ReLU functions in the same way as \cite{23}. The $k$-layer GCN is:
\begin{equation} \label{eq6}
    \hat{Y}=\text{softmax}(S_{sym}(\cdots (S_{sym}X\theta^{(1)})\cdots)\theta^{(k)}) \,.
\end{equation}
By collapsing the repeated multiplication with the $S_{sym}$ into a single matrix ${S_{sym}}^{k}$ and letting $\theta=\theta^{(1)}\cdots\theta^{(k)}$ to simplify the parameters and calculations, the model becomes:
\begin{equation} \label{eq7}
    \hat{Y}_{\text{SGC}}=\text{softmax}({S_{sym}}^{k}X\theta) \,,
\end{equation}
which is called Simple Graph Convolution (SGC) \cite{23}.

\subsection{Key Designs} \label{sec:design}
In this section, we propose four key designs that help improve the performance of learning from different types (homophily and heterophily) and scales (small, medium, and large) of networks: (D1) dilated convolution, (D2) multi-channel learning, (D3) self-attention score, and (D4) sign factor. For the sake of clarity, we describe each design in detail.

\mysection{(D1) Dilated convolution} \label{sec:dc}
To enlarge the receptive field without loss of resolution, Yu \etal \cite{27} propose a  dilated convolution as an alternative to stacking consecutive pooling layers. Their experiments on semantic segmentation tasks demonstrate the effectiveness of the convolution for performance improvement. We suppose that dilation also helps improve the performance of GNNs on graph learning tasks. Therefore, we introduce dilated convolution to GNNs. There are many possible ways to construct a dilated neighborhood. We find a set of neighboring nodes at different distances by applying the power of adjacency matrix (with or without self-loops) to construct geometric dilated neighborhoods. We define the geometric dilated neighborhoods by taking the adjacency matrix with self-loops as an example.
\begin{equation} \label{eq8}
    f_{d}^{(m)}=\{\hat{A}^{d_{m}+q_{0}}, \hat{A}^{q_{m}+d_{m}+q_{0}}, \cdots, \hat{A}^{(k-1)q_{m}+d_{m}+q_{0}}\} \,,
\end{equation}
where $d_{m}+q_{0}, q_{m}+d_{m}+q_{0}, \cdots, (k-1)q_{m}+d_{m}+q_{0}$ are a set of dilation factors and $\hat{A}^{(k-1)q_{m}+d_{m}+q_{0}}$ denotes the $(k-1)q_{m}+d_{m}+q_{0}$ power of $\hat{A}$. \figLabel \ref{fig:2} provides two ways to construct our geometric dilated neighborhoods. With the dilated neighborhoods, we define our geometric dilated convolution.
\begin{equation} \label{eq9}
\!H_{d}^{\!(m)\!}\!=\!\text{\!AGGR}\!\{\hat{A}^{\! d_{m}\!+\!q_{0}\!},\!\hat{A}^{\!q_{m}\!+\!d_{m}\!+\!q_{0}\!},\!\cdots\!,\!\hat{A}^{\!(k-1)\!q_{m}\!+\!d_{m}\!+\!q_{0}\!}\}\!X\!W ,
\end{equation}
where AGGR is an aggregation function that aggregates representations from the dilated neighborhoods in sum way, and $W$ denotes a learned weight matrix.

\begin{figure}[!htb]
    \centering
    \includegraphics[page=4, width=0.8\columnwidth]{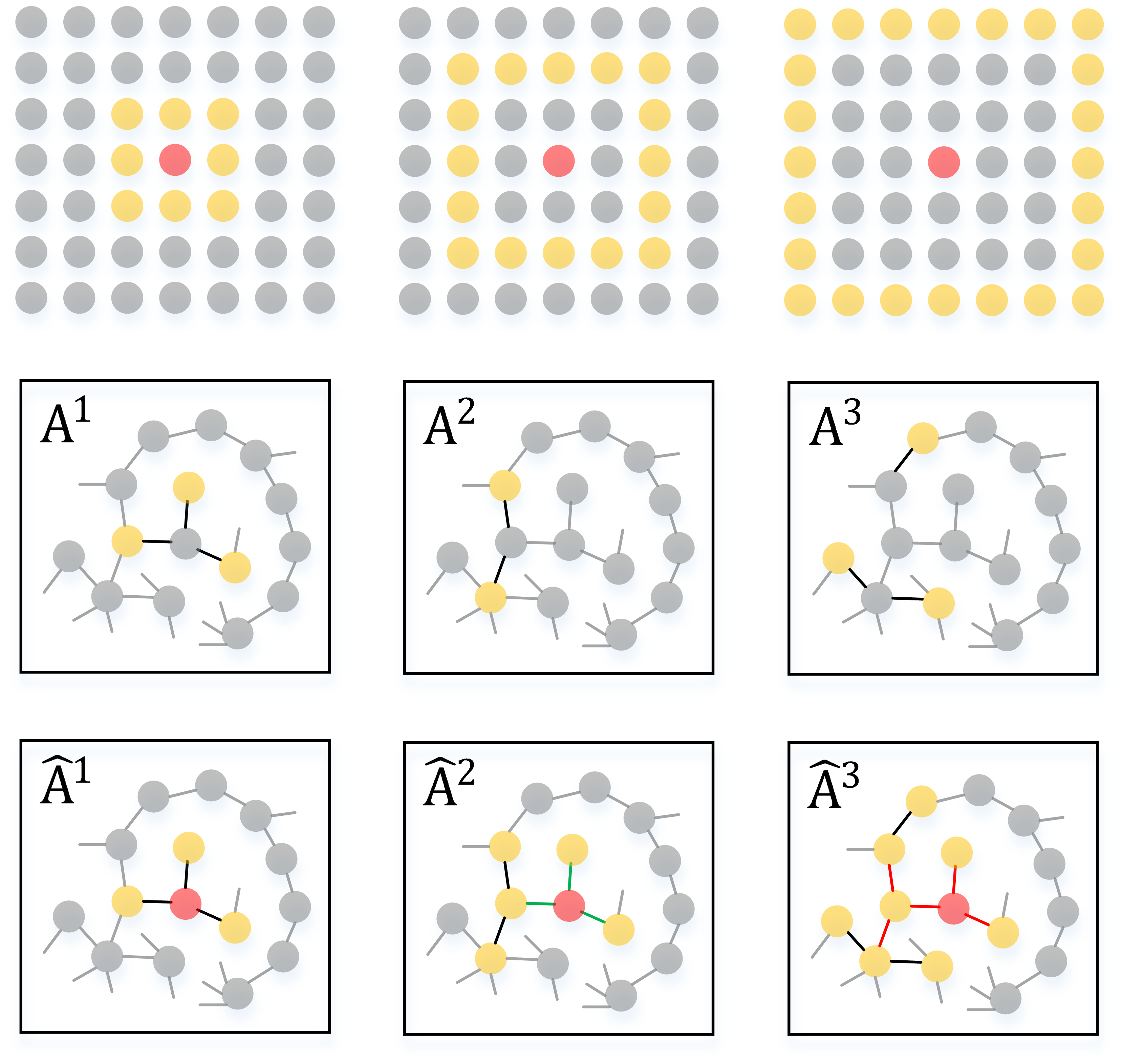}
    \caption{\textbf{Geometric dilated neighborhoods}. Visualization of dilated neighborhoods on a 2D image and on a graph. (\emph{top}) 2D neighborhoods with kernel size 3 and dilation factor 1, 2, 3 (left to right). (\emph{middle}) Geometric dilated neighborhoods with dilation factor 1, 2, 3 (left to right). (\emph{bottom}) Geometric dilated neighborhoods with self-loops and dilation factor 1, 2, 3 (left to right).}
%\vspace{-6pt}
\label{fig:2}
\end{figure}

\mysection{(D2) Multi-channel learning} \label{sec:mc}
In \eqLabel \eqref{eq9}, the single-channel convolution can effectively learn the interaction of neighboring nodes with different dilation factors. To further improve the interaction, we introduce multi-channel learning to our geometric dilated convolution.
\begin{equation} \label{eq10}
    H_{d}=aggregate(H_{d}^{(1)}, \cdots, H_{d}^{(m)}) \,,
\end{equation}
where $H_{d}^{(m)}$ is the geometric dilated convolution of the $m$-th channel, and $aggregate$ is an aggregation function that brings no extra parameters and is described in detail in \secLabel \ref{sec:gpcon}.

\mysection{(D3) Self-attention score} \label{sec:sa}
To analyze the contribution of the features of self-node to performance improvement compared to the features of its neighboring nodes, we introduce a self-attention score $\alpha\in R$ into the proposed model.
\begin{equation} \label{eq11}
H_{0d}^{(m)}=\text{AGGR}(\alpha\hat{A}^{0}XW, H_{d}^{(m)}) \,.
\end{equation}
We set the hyperparameter $\alpha$, which will be discussed in \secLabel \ref{sec:ablation}, according to different types and scales of networks. By setting $\alpha>0$, the final representation of each node consists of at least a fraction of $X$.

\mysection{(D4) Sign factor} \label{sec:sign}
Many GNNs such as APPNP and SGC invariably suppress the high-frequency signals \cite{48}. Thus, they may perform worse on heterophily graphs. To address the problem, we propose a sign factor $\beta$ $(\beta=1 \text{or} -1)$ to adapt to both homophily and heterophily graphs.
\begin{equation} \label{eq12}
H^{(m)}=\text{AGGR}(\alpha\hat{A}^{0}XW, \beta H_{d}^{(m)}) \,,
\end{equation}
where $\beta$ is a key another hyperparameter that controls the type of filter along with $\alpha$. We set $\beta=1$ and $\beta=-1$ for homophily and heterophily networks in most cases, respectively. If the effect of $W$ is ignored, the filter with $\beta=1$ mainly considers the low-frequency components while the filter with $\beta=-1$ pays more attention to the high-frequency components (see Theorem \ref{thm:3}).

\mysection{Summary of designs} \label{sec:design_sum}
To sum up, we believe that the first two designs are key techniques for boosting the performance of learning graph representations, and the latter two designs are effective strategies for fitting different types and scales of networks.

\subsection{The Overall Architecture} \label{sec:arch}
Current GNNs suffer from at least one of the fundamental limitations including over-smoothing, over-fitting, difficulty training, and strong homophily assumption. SGC, as an extremely efficient one-layer model, addresses the problems of over-fitting and difficulty training, by removing the nonlinearities between GCN layers and precomputing the fixed feature extraction \cite{23}. However, the analysis in \secLabel \ref{sec:sgc} shows that SGC fails to tackle the over-smoothing and is distant from optimal with heterophily graphs. To tackle the limitations, we propose a new simple and efficient one-layer model called GPNet. In GPNet, we provide the four key designs to fit a distribution under optimal complexity and accuracy for different types and scales of networks. With the designs, we design a simple convolution model with multi-channel geometric polynomials to learn graph representations. In output layer, we apply a softmax classifier to classify each node.

\subsection{Simple Convolution with Multi-channel Geometric Polynomials} \label{sec:gpcon}
Our simple convolution with multi-channel geometric polynomials to update node representations in three stages, \ie geometric adjacency matrix, geometric feature propagation, and geometric linear transformation.

\mysection{Geometric adjacency matrix} \label{sec:gam}
In GCNs, we use $S_{sym}$ to obtain the features of a node and its first-order neighboring nodes. However, $S_{sym}$ fails to obtain the features of high-order neighboring nodes. In SGC, we apply the $k$ power of $S_{sym}$ (${S_{sym}}^{k}$) in an one-layer model to capture the features of neighboring nodes that are $k$-hops away. As shown in \secLabel \ref{sec:sgc}, SGC has difficulty in capturing the features of neighboring nodes at long distances and fails to fit well a distribution with heterophily. To tackle the problems, we combine various geometric powers of the normalized adjacency matrix (with or without self-loops) based on different dilation factors into a single geometric adjacency matrix with six stages: (S1) features obtaining of own node and its neighboring nodes; (S2) features summation of neighboring nodes; (S3) features self-attention of own node; (S4) features sign of neighboring nodes; (S5) features summation of S3 and S4; and (S6) features aggregation. In the geometric adjacency matrix, we take the normalized adjacency matrix with self-loops $\hat{A}$ as an example to facilitate the analysis of the designed model.

\textbf{(S1) Features obtaining of own node and its neighboring nodes.}  With different dilation factors, we obtain the features of own node and its neighboring nodes at various distances via a series of functions of $\hat{A}$.
\begin{equation} \label{eq13}
f(\hat{A})=(f_{1}(\hat{A}), f_{2}(\hat{A}), \cdots, f_{m}(\hat{A})) \,,
\end{equation}
where $f_{m}(\hat{A})$ is the geometric polynomial of $\hat{A}$ for $i\neq 0$, which is defined as:
\begin{equation} \label{eq14}
f_{m}(\hat{A})=
\begin{cases}
\hat{A}^{i},& \text{$i=0$}\\
\hat{A}^{(i-1)q_{m}+d_{m}+q_{0}},& \text{$i=1, 2, \cdots, k$}
\end{cases}
\end{equation}
where $m$ and $k$ are the number of channels and the number of terms, respectively. Here $q_{0}\geq 0, q_{m}\geq 1$, and $d_{m}\geq 0$ are the first item coefficient, common ratios, and neighborhood coefficients of $f_{m}(\hat{A})$, respectively. In \eqLabel \eqref{eq14}, we can fix the value of $m$ to obtain the geometric adjacency matrices of different dilation factors. For example when $m=1$, we can get a set of adjacency matrices with dilation factor $d_{1}+q_{0}, q_{1}+d_{1}+q_{0}, \cdots, (k-1)q_{1}+d_{1}+q_{0}$.

\textbf{(S2) Features summation of neighboring nodes.} After the stage S1, we sum the features of neighboring nodes using the geometric polynomials of adjacency matrices with various dilation factors.
\iffalse
\begin{equation} \label{eq15}
\begin{cases}
f_{1}\!(\!\hat{A}\!)_{\!sum}\!=\!\hat{A}^{d_{1}+q_{0}}\!+\!\hat{A}^{q_{1}+d_{1}\!+\!q_{0}}\!+\!\!\cdots\!\!+\!\hat{A}^{(\!k-1\!)q_{1}\!+\!d_{1}\!+\!q_{0}}\\
f_{2}\!(\!\hat{A}\!)_{\!sum}\!=\!\hat{A}^{d_{2}+q_{0}}\!+\!\hat{A}^{q_{2}+d_{2}\!+\!q_{0}}\!+\!\!\cdots\!\!+\!\hat{A}^{(\!k-1\!)q_{2}\!+\!d_{2}\!+\!q_{0}}\\
\vdots \\
f_{m}\!(\!\hat{A}\!)_{\!sum}\!\!=\!\!\hat{A}^{\!d_{m}\!+\!q_{0}}\!\!+\!\!\hat{A}^{\!q_{m}\!+\!d_{m}\!+\!q_{0}}\!+\!\!\cdots\!\!+\!\!\hat{A}^{(\!k-1\!)\!q_{m}\!+\!d_{m}\!+\!q_{0}}
\end{cases}
\end{equation}
\fi
\begin{small}
\begin{equation} \label{eq15}
\begin{split}
f_{1}(\hat{A})_{sum}&=\hat{A}^{d_{1}+q_{0}}+\hat{A}^{q_{1}+d_{1}+q_{0}}+\cdots+\hat{A}^{(k-1)q_{1}+d_{1}+q_{0}},\\
f_{2}(\hat{A})_{sum}&=\hat{A}^{d_{2}+q_{0}}+\hat{A}^{q_{2}+d_{2}+q_{0}}+\cdots+\hat{A}^{(k-1)q_{2}+d_{2}+q_{0}},\\
\vdots \\
f_{m}(\hat{A})_{sum}&=\hat{A}^{d_{m}+q_{0}}+\hat{A}^{q_{m}+d_{m}+q_{0}}+\cdots+\hat{A}^{(k-1)q_{m}+d_{m}+q_{0}}.
\end{split}
\end{equation}
\end{small}

\textbf{(S3) Features self-attention of own node.} We apply a self-attention score $\alpha \in R$ to adjust the contribution of node's own features.
\begin{equation} \label{eq16}
S^{0}=\alpha\hat{A}^{0} \,.
\end{equation}

\textbf{(S4) Features sign of neighboring nodes.} We propose a sign factor $\beta$ $(\beta=1 \text{or} -1)$ that is a key design of fitting a distribution on both homophily and heterophily graphs. By applying $\beta$ to \eqLabel \eqref{eq16}, we have:
\begin{small}
\begin{equation} \label{eq17}
\begin{split}
\beta f_{1}(\hat{A})_{sum}&=\beta\left(\hat{A}^{d_{1}+q_{0}}+\hat{A}^{q_{1}+d_{1}+q_{0}}+\cdots+\hat{A}^{(k-1)q_{1}+d_{1}+q_{0}}\right),\\
\beta f_{2}(\hat{A})_{sum}&=\beta(\hat{A}^{d_{2}+q_{0}}+\hat{A}^{q_{2}+d_{2}+q_{0}}+\cdots+\hat{A}^{(k-1)q_{2}+d_{2}+q_{0}}),\\
\vdots \\
\beta f_{m}(\hat{A})_{sum}&=\beta(\hat{A}^{d_{m}+q_{0}}\!+\!\hat{A}^{q_{m}+d_{m}+q_{0}}\!+\!\cdots\!+\!\hat{A}^{(k-1)q_{m}+d_{m}+q_{0}}).
\end{split}
\end{equation}
\end{small}

\textbf{(S5) Features summation of S3 and S4.} After the stages of S3 and S4, we sum the features of S3 and S4 in \eqLabel \eqref{eq18}.
\begin{small}
\begin{equation} \label{eq18}
\begin{split}
g_{1}(\hat{A})_{sum} &= S^{0}+\beta f_{1}(\hat{A})_{sum} \\
&= \alpha\hat{A}^{0}\!+\!\beta(\hat{A}^{d_{1}+q_{0}}\!+\!\hat{A}^{q_{1}+d_{1}+q_{0}}\!+\!\!\cdots\!\!+\!\hat{A}^{(k-1)q_{1}+d_{1}+q_{0}}),\\
g_{2}(\hat{A})_{sum} &= S^{0}+\beta f_{2}(\hat{A})_{sum} \\
&= \alpha\hat{A}^{0}\!+\!\beta(\hat{A}^{d_{2}+q_{0}}\!+\!\hat{A}^{q_{2}+d_{2}+q_{0}}\!+\!\!\cdots\!\!+\!\hat{A}^{(k-1)q_{2}+d_{2}+q_{0}}),\\
\vdots \\
g_{m}(\hat{A})_{sum} &= S^{0}+\beta f_{m}(\hat{A})_{sum} \\
&= \!\alpha\hat{A}^{0}\!\!+\!\beta(\hat{A}^{d_{m}\!+q_{0}}\!\!+\!\hat{A}^{q_{m}\!+d_{m}\!+q_{0}}\!\!+\!\!\cdots\!\!+\!\hat{A}^{(k-1)q_{m}\!+d_{m}\!+q_{0}}).\\
\end{split}
\end{equation}
\end{small}

\textbf{(S6) Features aggregation.} Finally, we aggregate the features of $g_{1}(\hat{A})_{sum}, g_{2}(\hat{A})_{sum}, \cdots, g_{m}(\hat{A})_{sum}$.
\begin{equation} \label{eq19}
\hat{S}_{adj}=aggregate(g_{1}(\hat{A})_{sum}, g_{2}(\hat{A})_{sum}, \cdots, g_{m}(\hat{A})_{sum}) \,,
\end{equation}
where $aggregate$ is an aggregation function. Instead of  the aggregation schemes of increasing parameters as in concatenation and LSTM-attention \cite{20}, we explore four aggregation approaches without additional parameters, including Max-FP, Min-FP, Avg-FP, and Sum-FP. Our aggregation approaches are all element-wise operations, which are defined as:
\begin{equation} \label{eq20}
\begin{split}
\hat{S}_{adj}&=\text{Max-FP}(g_{1}(\hat{A})_{sum}, g_{2}(\hat{A})_{sum}, \cdots, g_{m}(\hat{A})_{sum}) \\
&=max(g_{1}(\hat{A})_{sum}, g_{2}(\hat{A})_{sum}, \cdots, g_{m}(\hat{A})_{sum}) \,,
\end{split}
\end{equation}
\begin{equation} \label{eq21}
\begin{split}
\hat{S}_{adj}&=\text{Min-FP}(g_{1}(\hat{A})_{sum}, g_{2}(\hat{A})_{sum}, \cdots, g_{m}(\hat{A})_{sum}) \\
&=min(g_{1}(\hat{A})_{sum}, g_{2}(\hat{A})_{sum}, \cdots, g_{m}(\hat{A})_{sum}) \,,
\end{split}
\end{equation}
\begin{equation} \label{eq22}
\begin{split}
\hat{S}_{adj}&=\text{Avg-FP}(g_{1}(\hat{A})_{sum}, g_{2}(\hat{A})_{sum}, \cdots, g_{m}(\hat{A})_{sum}) \\
&=avg(g_{1}(\hat{A})_{sum}, g_{2}(\hat{A})_{sum}, \cdots, g_{m}(\hat{A})_{sum}) \\
&=\frac{1}{m}\sum\nolimits_{r=1}^{m} g_{r}(\hat{A})_{sum} \,,
\end{split}
\end{equation}
\begin{equation} \label{eq23}
\begin{split}
\hat{S}_{adj}&=\text{Sum-FP}(g_{1}(\hat{A})_{sum}, g_{2}(\hat{A})_{sum}, \cdots, g_{m}(\hat{A})_{sum}) \\
&=sum(g_{1}(\hat{A})_{sum}, g_{2}(\hat{A})_{sum}, \cdots, g_{m}(\hat{A})_{sum}) \\
&=\sum\nolimits_{r=1}^{m} g_{r}(\hat{A})_{sum} \,.
\end{split}
\end{equation}
\begin{proposition}\label{prop:1}
    Geometric adjacency matrix is an operation of permutation invariance.
\end{proposition}
\begin{proof}
    In Max-FP, Min-FP, Avg-FP, and Sum-FP, we have
    \begin{align*}
        \text{Max-FP}(g_{1}(\hat{A})_{sum}, g_{2}(\hat{A})_{sum}, \cdots, g_{m}(\hat{A})_{sum})=\\ \text{Max-FP}(g_{m}(\hat{A})_{sum}, g_{m-1}(\hat{A})_{sum}, \cdots, g_{1}(\hat{A})_{sum}),
    \end{align*}
    \begin{align*}
        \text{Min-FP}(g_{1}(\hat{A})_{sum}, g_{2}(\hat{A})_{sum}, \cdots, g_{m}(\hat{A})_{sum})=\\ \text{Min-FP}(g_{m}(\hat{A})_{sum}, g_{m-1}(\hat{A})_{sum}, \cdots, g_{1}(\hat{A})_{sum}),
    \end{align*}
    \begin{align*}
        \text{Avg-FP}(g_{1}(\hat{A})_{sum}, g_{2}(\hat{A})_{sum}, \cdots, g_{m}(\hat{A})_{sum})=\\ \text{Avg-FP}(g_{m}(\hat{A})_{sum}, g_{m-1}(\hat{A})_{sum}, \cdots, g_{1}(\hat{A})_{sum}),
    \end{align*}
    \begin{align*}
        \text{Sum-FP}(g_{1}(\hat{A})_{sum}, g_{2}(\hat{A})_{sum}, \cdots, g_{m}(\hat{A})_{sum})=\\ \text{Sum-FP}(g_{m}(\hat{A})_{sum}, g_{m-1}(\hat{A})_{sum}, \cdots, g_{1}(\hat{A})_{sum}),
    \end{align*}
    thus finishing the proof.
\end{proof}
\begin{proposition}\label{prop:2}
    Geometric adjacency matrix is an operation of preserving graph topology.
\end{proposition}
\begin{proof}
    Let $\hat{A}\in R^{n\times n}$, then
    \begin{small}
    \begin{align*}
    \begin{split}
        g&_{1}(\hat{A})_{sum}\!=\!\alpha \hat{A}^{0}\!\!+\!\beta(\hat{A}^{d_{1}\!+\!q_{0}}\!\!+\!\hat{A}^{q_{1}\!+\!d_{1}\!+\!q_{0}}\!\!+\!\!\cdots\!\!+\!\hat{A}^{(k-1)q_{1}\!+\!d_{1}\!+\!q_{0}})\!\!\in\!\! R^{n\!\times\! n},\\
        g&_{2}(\hat{A})_{sum}\!=\!\alpha \hat{A}^{0}\!\!+\!\beta(\hat{A}^{d_{2}\!+\!q_{0}}\!\!+\!\hat{A}^{q_{2}\!+\!d_{2}\!+\!q_{0}}\!\!+\!\!\cdots\!\!+\!\hat{A}^{(k-1)q_{2}\!+\!d_{2}\!+\!q_{0}})\!\!\in\!\! R^{n\!\times\! n},\\
        \vdots \\
        g&_{m}\!(\!\hat{A})_{\!sum}\!\!=\!\alpha \hat{A}^{0}\!\!+\!\!\beta(\!\hat{A}^{d_{m}\!+\!q_{0}}\!\!+\!\!\hat{A}^{q_{m}\!+\!d_{m}\!+\!q_{0}}\!\!+\!\!\cdots\!\!+\!\!\hat{A}^{(\!k-1\!)q_{m}\!+\!d_{m}\!+\!q_{0}}\!)\!\!\in\!\! R^{n\!\times\! n},\\
    \end{split}
    \end{align*}
    and $\hat{S}_{adj}\!=\!aggregate(g_{1}(\hat{A})_{sum}, g_{2}(\hat{A})_{sum}, \!\!\cdots\!\!, g_{m}(\hat{A})_{sum})\!\!\in\!\! R^{n\times n}$. Equations \eqref{eq20}-\eqref{eq23} show that the aggregation schemes of Max-FP, Min-FP, Avg-FP, and Sum-FP are element-wise operations. Therefore, the spatial location of $\hat{S}_{adj}$ is the same as that of $\hat{A}$, which finishes the proof.
    \end{small}
\end{proof}

\mysection{Geometric feature propagation} \label{sec:gfp}
Geometric feature propagation is what distinguishes our model from many GNNs such as GCN and SGC. The node feature matrix $X$ is averaged with the features of neighboring nodes constructed by the geometric adjacency matrix.
\begin{equation} \label{eq24}
\bar{H}=\hat{S}_{adj}X \,.
\end{equation}
This step encourages similar predictions among locally connected nodes on homophily graphs and improves the ability of distinguishing the features of nodes in different clusters on heterophily graphs. In addition, we can think of $\bar{H}$ as a feature pre-processing step before training due to the parameterless computations of $\bar{H}$.

\mysection{Geometric linear transformation} \label{sec:glt}
After the geometric feature propagation, we implement a geometric linear transformation by $\bar{H}W$.
\begin{equation} \label{eq25}
H=\bar{H}W \,.
\end{equation}

Based on Equations \eqref{eq24}-\eqref{eq25}, we conclude that the proposed simple convolution with multi-channel geometric polynomials takes the following form.
\begin{equation} \label{eq26}
H=\hat{S}_{adj}XW \,.
\end{equation}
\begin{algorithm}[tb]
\caption{Simple convolution with multi-channel geometric polynomials}
\label{alg:algorithm}
\textbf{Inputs}: a normalized adjacency matrix with self-loops $\hat{A}$, an initial feature matrix $X$, and a learned weight matrix $W$.
\begin{algorithmic}[1] %[1] enables line numbers
\STATE \textbf{for} $i=0$ \textbf{do}
\STATE \quad $f_{m}^{(i)}(\hat{A})=\hat{A}^{i}$.
\STATE \textbf{end for}
\STATE \textbf{for} $i=1$ \textbf{to} $k$ \textbf{do}
\STATE \quad $f_{m}^{(i)}(\hat{A})=\hat{A}^{(i-1)q_{m}+d_{m}+q_{0}}$.
\STATE \textbf{end for}
\STATE Calculate geometric adjacency matrix $\hat{S}_{adj}$ via the six stages of (S1)-(S6).
\STATE $\bar{H} = \hat{S}_{adj}X$.
\STATE $H = \bar{H}W=\hat{S}_{adj}XW$.
\STATE \textbf{return} $H$
\end{algorithmic}
\end{algorithm}

Our algorithm (with self-loops) is summarized in Algorithm~\ref{alg:algorithm}. Our convolution model aims to fit a distribution on both homophily and heterophily graphs via four designs of (D1)-(D4). In GCN and SGC, the convolution acts as a low-pass filter, which is suitable for homophily graphs but may not work well with heterophily. However, our convolution adapts to different types of networks by utilizing signals of different frequencies.

\subsection{Output Layer} \label{sec:out}
In output layer, we follow Kipf and Welling \cite{1} to predict the label of nodes using a $\text{softmax}$ classifier and a cross entropy loss. The output prediction $Y_{\text{GPNet}}$ is:
\begin{equation} \label{eq27}
Y_{\text{GPNet}} = \text{softmax}(H)=\text{softmax}(\hat{S}_{adj}XW) \,.
\end{equation}

\section{Theoretical Analysis}
\label{sec:analysis}

\subsection{Analysis of SGC}
\label{sec:sgc}
SGC, as a simplified linear model, has the same increased receptive field of $k$-layer GNN, which can aggregate the information of neighbor nodes that are $k$-hops away. We increase the receptive field and aggregate the features of neighbor nodes at more distant distances by taking a larger value of $k$. However, taking a larger value of $k (k\geq 3)$ makes the output features of nodes in the same cluster and different clusters similar and hurts the classification task \cite{26}. This shows that SGC with large $k$ value may suffer from over-smoothing. We propose Theorem \ref{thm:1} for further proof.
\begin{theorem}\label{thm:1}
    For an undirected graph, SGC corresponds to a $k$-step random walk considering self-loops. As $k$ tends to infinity, ${S_{sym}}^{k}$  converges to a stationary distribution $\pi=\frac{\tilde{D}}{2e+n}$, where $e$ is the number of edges in the graph.
\end{theorem}
\begin{proof}
In \eqLabel \eqref{eq7}, we have ${S_{sym}}^{k}X\theta=(\tilde{D}^{-\frac{1}{2}}\tilde{A}\tilde{D}^{-\frac{1}{2}})^{k}X\theta$. Let transition matrix with self-loops $\tilde{P}=\tilde{A}\tilde{D}^{-1}$, then
\begin{equation} \label{eq28}
\begin{split}
(\tilde{D}^{-\frac{1}{2}}\tilde{A}\tilde{D}^{-\frac{1}{2}})^{k} &= \tilde{D}^{-\frac{1}{2}}\tilde{A}\tilde{D}^{-\frac{1}{2}} \cdot \tilde{D}^{-\frac{1}{2}}\tilde{A}\tilde{D}^{-\frac{1}{2}} \cdots \tilde{D}^{-\frac{1}{2}}\tilde{A}\tilde{D}^{-\frac{1}{2}} \\
&= \tilde{D}^{-\frac{1}{2}}(\tilde{A}\tilde{D}^{-1})(\tilde{A}\tilde{D}^{-1}) \cdots (\tilde{A}\tilde{D}^{-1}) \cdot \tilde{A}\tilde{D}^{-\frac{1}{2}} \\
&= \tilde{D}^{-\frac{1}{2}}(\tilde{A}\tilde{D}^{-1})^{k-1} \cdot \tilde{A}\tilde{D}^{-\frac{1}{2}} \\
&=\tilde{D}^{-\frac{1}{2}}(\tilde{A}\tilde{D}^{-1})^{k-1} \cdot \tilde{A}\tilde{D}^{-\frac{1}{2}} \cdot \tilde{D}^{-\frac{1}{2}} \cdot \tilde{D}^{\frac{1}{2}} \\
&= \tilde{D}^{-\frac{1}{2}}(\tilde{A}\tilde{D}^{-1})^{k-1} \cdot (\tilde{A}\tilde{D}^{-1}) \cdot \tilde{D}^{\frac{1}{2}} \\
&= \tilde{D}^{-\frac{1}{2}}(\tilde{A}\tilde{D}^{-1})^{k} \tilde{D}^{\frac{1}{2}} \\
&= \tilde{D}^{-\frac{1}{2}}\tilde{P}^{k} \tilde{D}^{\frac{1}{2}} \,.
\end{split}
\end{equation}
Therefore, ${S_{sym}}^{k}$ converges to a stationary distribution $\pi=\frac{\tilde{D}}{2e+n}$ as $k$ goes to infinity.
\end{proof}

Theorem \ref{thm:1} shows that the distribution is only related to the graph structure while forgetting the initial features, so SGC suffers from over-smoothing under the distribution.
In SGC, the convolution corresponds to a fixed non-negative low-pass filter \cite{23} that retains the commonality while ignoring the difference, thus making the learned representations of connected nodes similar. Since the connected nodes of homophily and heterophily graphs usually belong to same cluster and different labels, respectively, SGC is suitable for homophily graphs but may not perform well with heterophily graphs.
\begin{theorem}\label{thm:2}
    Many GNNs such as SGC are equivalent to a low-pass filter, and can also be regarded as a high-pass filter.
\end{theorem}
\begin{proof}
As mentioned in \eqLabel \eqref{eq7}, the resulting classifier of SGC is $\hat{Y}_{\text{SGC}}=\text{softmax}({S_{sym}}^{k}X\theta)$, then the propagation matrix ${S_{sym}}^{k}$ corresponds to filter coefficients $\hat{g}(\tilde{\lambda}_{i})=(1-\tilde{\lambda}_{i})^{k}$. Thus, ${S_{sym}}^{k}$ acts as a low-pass filter. Let $\theta_{hp}=-\theta$, then $\hat{Y}_{\text{SGC}}=\text{softmax}(-{S_{sym}}^{k}X\theta_{hp})$. Therefore, the filter coefficients are $\hat{g}(\tilde{\lambda}_{i})=-(1-\tilde{\lambda}_{i})^{k}$, which can be regarded as a high-pass filter.
\end{proof}
\subsection{Analysis of GPNet}
We follow Defferrard \etal \cite{29} to introduce the graph convolution between signal $x$ and filter $g$.
\begin{equation} \label{eq29}
g \ast x = U((U^{T}g)\odot (U^{T}x)) = Udiag(\hat{g}_{1}, \cdots, \hat{g}_{n})U^{T}x),
\end{equation}
where $diag(\hat{g}_{1}, \cdots, \hat{g}_{n})$ denotes a diagonal matrix of spectral filter coefficients. GCN simplifies the graph filter $diag(\hat{g}_{1}, \cdots, \hat{g}_{n})$ as $I - \Lambda$.

In \eqLabel \eqref{eq1}, the propagation matrix is $S_{sym}=\tilde{D}^{-\frac{1}{2}}\tilde{A}\tilde{D}^{-\frac{1}{2}}$, then the augmented normalized Laplacian matrix is defined as $\tilde{L}_{sym} = \tilde{D}^{-\frac{1}{2}}(\tilde{D}-\tilde{A})\tilde{D}^{-\frac{1}{2}}=I - S_{sym}$. Therefore, the filter coefficients are $\hat{g}(\tilde{\lambda}_{i})=(1-\tilde{\lambda}_{i})$, where $\tilde{\lambda}_{i}$ denote the eigenvalues of $\tilde{L}_{sym}$. In GPNet, the filter coefficients are $\hat{g}(\tilde{\lambda}_{i})=aggregate(\hat{g}_{1}(\tilde{\lambda}_{i}), \hat{g}_{2}(\tilde{\lambda}_{i}), \cdots, \hat{g}_{m}(\tilde{\lambda}_{i}))$, where $\hat{g}_{m}(\tilde{\lambda}_{i})=\alpha+\beta((1-\tilde{\lambda}_{i})^{d_{m}+q_{0}}+(1-\tilde{\lambda}_{i})^{q_{m}+d_{m}+q_{0}}+\cdots+(1-\tilde{\lambda}_{i})^{(k-1)q_{m}+d_{m}+q_{0}})$.
 
\begin{theorem}\label{thm:3}
    GPNet can approximate various graph filters.
\end{theorem}
\begin{proof}
Let $m=1$, then the filter coefficients of GPNet are $\hat{g}(\tilde{\lambda}_{i})=\alpha+\beta((1-\tilde{\lambda}_{i})^{d_{1}+q_{0}}+(1-\tilde{\lambda}_{i})^{q_{1}+d_{1}+q_{0}}+\cdots+(1-\tilde{\lambda}_{i})^{(k-1)q_{1}+d_{1}+q_{0}})$. As shown in \secLabel \ref{sec:gpcon}, we have $q_{0}\geq 0, q_{1}\geq 1$, and $d_{1}\geq 0$. We set $\alpha>0$ and $\beta=1$, then the filter coefficients become $\hat{g}(\tilde{\lambda}_{i})=\alpha+(1-\tilde{\lambda}_{i})^{d_{1}+q_{0}}+(1-\tilde{\lambda}_{i})^{q_{1}+d_{1}+q_{0}}+\cdots+(1-\tilde{\lambda}_{i})^{(k-1)q_{1}+d_{1}+q_{0}}$, which acts as a low-pass graph filter. We set $\alpha\geq0$ and $\beta=-1$, then the filter coefficients are $\hat{g}(\tilde{\lambda}_{i})=\alpha-(1-\tilde{\lambda}_{i})^{d_{1}+q_{0}}-(1-\tilde{\lambda}_{i})^{q_{1}+d_{1}+q_{0}}-\cdots-(1-\tilde{\lambda}_{i})^{(k-1)q_{1}+d_{1}+q_{0}}$, which achieves a high-pass graph filter. We set $\alpha\geq0$, $\beta=1$, $k=1$, $d_{1}=0$, and $q_{0}=0$, then the filter coefficients become $\hat{g}(\tilde{\lambda}_{i})=\alpha+1$, \ie an all-pass graph filter.
\end{proof}
Theorem \ref{thm:3} shows that GPNet can pay attention to signals of different frequencies to fit different types of networks.

\mysection{GPNet without over-smoothing} Equations \eqref{eq18}, \eqref{eq19}, and \eqref{eq24} suggest that the proposed geometric feature propagation carries the features from both input node and its nodes of geometric dilated neighborhoods by setting $\alpha>0$. This ensures that GPNet can prevent the over-smoothing even if the number of terms goes to infinity. More precisely, when $\alpha=0$, the feature propagation considers the features of these neighboring nodes $d_{m}+q_{0}, q_{m}+d_{m}+q_{0}, \cdots, (k-1)q_{m}+d_{m}+q_{0}$. Even if $k$ tends to infinity, the feature propagation always includes the features that are $d_{m}+q_{0}$ hops away. Similarly, the feature propagation always includes the features of input node and its neighboring nodes that are $d_{m}+q_{0}$ hops away when $\alpha>0$. Therefore, our GPNet can avoid the over-smoothing even with $k$ going to infinity.

\begin{table*}[!tb]
	\centering
	\caption{\textbf{Dataset statistics}. Net. denotes Network.}
	\label{tab:1}
	\vspace{-0.2cm}
	\setlength{\tabcolsep}{1.43mm}{
		\begin{tabular}{llllllllll}
			\toprule
			Dataset  & Cora & Citeseer & Pubmed & Reddit & Cornell & Texas & Wisconsin & Chameleon & Squirrel \\
			\midrule
			Nodes     & 2,708 & 3,327  & 19,717 & 233K & 183 & 183 & 251 & 2,277 & 5,201 \\
			Features & 1,433 & 3,703  & 500 & 602 & 1,703 & 1,703 & 1,703 & 2,325 & 2,089    \\
			Classes   & 7 & 6 &  3 & 41 & 5 &  5 & 5 & 5 &  5    \\
			Edges & 5,429 & 4,732  & 44,338  & 11.6M & 295  & 309 & 499 & 36,101  & 198K    \\
			Network & Homophily & Homophily & Homophily & Homophily & Heterophily & Heterophily & Heterophily & Heterophily  & Heterophily \\
			Type & Citation Net. & Citation Net.  & Citation Net. & Social Net.  & Webpage Net.  & Webpage Net. & Wikipedia Net. & Wikipedia Net. & Wikipedia Net. \\
			\bottomrule
	\end{tabular}}
\end{table*}
\subsection{Analysis of Relation between GPNet and Representative Models}
\mysection{Relation of GPNet and SGC} SGC with 2-hops is:
\begin{equation} \label{eq30}
    Y_{\text{SGC}}=\text{softmax}({S_{sym}}^{2}X\theta) \,.
\end{equation}
In GPNet, let $\alpha=0$, $\beta=1$, $k=1$, and $d_{1}+q_{0}=2$, then $\hat{S}_{adj}=g_{1}(\hat{A})_{sum}=\hat{A}^{2}$. Thus the model becomes:
\begin{equation} \label{eq31}
Y_{\text{GPNet}} = \text{softmax}(\hat{A}^{2}XW) \,.
\end{equation}
As $\hat{A}$ and $W$ are the augmented normalized adjacency matrix and trainable weight matrix respectively equivalent to $S_{sym}$ and $\theta$ in \eqLabel \eqref{eq30}. Therefore, SGC can be reduced to a special case of GPNet.

\mysection{Relation of GPNet and gfNN} In gfNN~\cite{add_1}, the output prediction $Y_{\text{gfNN}}$ is:
\begin{equation} \label{eq32}
    Y_{\text{gfNN}}=\text{softmax}(\sigma(\hat{A}^{2}XW_{1})W_{2}) \,,
\end{equation}
where $\sigma$ is an activation function, $W_{1}$ and $W_{2}$ are different weight matrices. If we remove the activation function or use an identity activation function, then $Y_{\text{gfNN}}$ becomes:
\begin{equation} \label{eq33}
Y_{\text{gfNN}} = \text{softmax}(\hat{A}^{2}XW_{1}W_{2}) \,.
\end{equation}
Let $W=W_{1}W_{2}$, then gfNN can be regarded as a special case of GPNet.

\mysection{Relation of GPNet and MLP} In GPNet, let $m=1$, $k=1$, $d_{1}=0$, and $q_{0}=0$, then $\hat{S}_{adj}=g_{1}(\hat{A})_{sum}=(\alpha+\beta)\hat{A}^{0}$. By $W_{mlp}=(\alpha+\beta)W$, the model takes the following form:
\begin{equation} \label{eq34}
Y_{\text{GPNet}} = \text{softmax}((\alpha+\beta)\hat{A}^{0}XW)=\text{softmax}(XW_{mlp}) \,.
\end{equation}
Thus, MLP with a layer can be regarded as a special case of GPNet.

\subsection{Analysis of Time Complexity and Parameters} \label{sec:time}
As described in \secLabel \ref{sec:gpcon}, the computation of $\bar{H} = \hat{S}_{adj}X$ (geometric feature propagation) requires no weight. Therefore, by precomputing a parameter-free feature extraction $\bar{H}$, the entire training time and parameters of GPNet are the same as a particularly efficient SGC.

\section{Experiments} \label{sec:experiments}
We evaluate the performance of GPNet against the competitors, with the goal of answering the following research questions:

\textbf{Q1} Can GPNet tackle the over-smoothing when obtaining the features of neighboring nodes with large k-hops?

\textbf{Q2} Can GPNet perform well on heterophily graphs without trading off accuracy in homophily?

\textbf{Q3} Can GPNet work well on large-scale graphs when many GNNs fail to deal with such data?

\textbf{Q4} How does GPNet compare to the state-of-the-art (SOTA) models for different types of networks?

\textbf{Q5} How does GPNet compare to baselines in terms of the complexity and parameters?
\subsection{Dataset}
\label{sec:data}
We evaluate the performance of GPNet against the competitors on multiple real-world datasets chosen from benchmarks commonly used on node classification tasks. More specifically, for semi-supervised node classification, we use three homophilic datasets Cora, Citeseer, and Pubmed \cite{49} on public splits \cite{1}. For full-supervised node classification, we also use the three homophilic datasets and five heterophily datasets including webpage graphs Cornell, Texas, and Wisconsin  \cite{45}, and Wikipedia graphs Chameleon and Squirrel \cite{50}. For these datasets of the full-supervised node classification tasks, we follow Pei \etal \cite{45} to split the nodes of training, validation and testing. For inductive learning, we include a large graph Reddit \cite{32}. We summarize the statistics of these datasets in \tblLabel \ref{tab:1}.

\begin{table}[!htb]
\centering
\caption{\textbf{The classification performance (\%) over the methods on semi-supervised node classification tasks}. Avg. R. denotes Average Rank.}
\label{tab:2}
\vspace{-0.2cm}
\setlength{\tabcolsep}{1mm}{
\begin{tabular}{l|lll|ll}
\toprule
Method  & Pubmed & Citeseer & Cora & Avg. R. & Average \\
\midrule
GCN     & 79.0 & 70.3  & 81.5 & 9.83 & 76.9 \\
GAT & 79.0$ \pm 0.3 $ & 72.5$ \pm 0.7 $  & 83.0$ \pm 0.7 $ & 7.67 & 78.2    \\
SGC  & 78.9$ \pm 0.0 $ & 71.9$ \pm 0.1 $  & 81.0$ \pm 0.0 $ & 10.17 & 77.3    \\
MixHop-learn  & 80.8$ \pm 0.6 $ & 71.4$ \pm 0.8 $  & 81.9$ \pm 0.4 $ & 7 & 78.0    \\
APPNP  & 80.1 & 71.8  & 83.3 & 6.5 & 78.4 \\
ARMA  & 78.9 & 72.5  & 83.4 & 7.33 & 78.3 \\
IncepGCN+DropEdge  & 79.5 & 72.7  & 83.5 & 4.17 & 78.6 \\
JKNet  & 78.1 & 69.8  & 81.1 & 11.67 & 76.3 \\
JKNet+DropEdge  & 79.2 & 72.6  & 83.3 & 6.17 & 78.4 \\
GCNII  & 80.2$ \pm 0.4 $ & 73.4$ \pm 0.6 $  & \textbf{85.5$\pm$0.5} & 2 & 79.7 \\
FAGCN  & 79.4$ \pm 0.3 $ & 72.7$ \pm 0.8 $  & 84.1$ \pm 0.5 $ & 4.17 & 78.7 \\
GPNet (ours)  & \textbf{81.5$\pm$0.0} & \textbf{74.8$\pm$0.1}  & 84.6$ \pm 0.1 $ & \textbf{1.33} & \textbf{80.3} \\
\bottomrule
\end{tabular}}
\end{table}
\begin{table*}[!htb]
\centering
\caption{\textbf{The classification performance (\%) over the methods on full-supervised node classification tasks}. Avg. R. denotes Average Rank.}
\label{tab:3}
\vspace{-0.2cm}
\setlength{\tabcolsep}{2.2mm}{
\begin{tabular}{l|lllll|lll|ll}
\toprule
Method  & Chameleon & Squirrel & Wisconsin	& Cornell & Texas & Pubmed & Citeseer & Cora & Avg. R. & Average \\
\midrule
MLP & 46.21$ \pm 3.0 $ & 28.77$ \pm 1.6 $ & 85.29$ \pm 3.3 $ & 81.89$ \pm 6.4 $ & 80.81$ \pm 4.8 $ & 87.16$ \pm 0.4 $ & 74.02$ \pm 1.9 $ & 75.69$ \pm 2.0 $ & 14.13 & 69.98 \\
GCN & 64.82$ \pm 2.2 $ & 53.43$ \pm 2.0 $ & 51.76$ \pm 3.1 $ & 60.54$ \pm 5.3 $ & 55.14$ \pm 5.2 $ & 88.42$ \pm 0.5 $ & 76.50$ \pm 1.4 $ & 86.98$ \pm 1.3 $ & 12.88 & 67.20 \\
GAT & 60.26$ \pm 2.5 $ & 40.72$ \pm 1.6 $ & 49.41$ \pm 4.1 $ & 61.89$ \pm 5.1 $ & 52.16$ \pm 6.6 $ & 86.33$ \pm 0.5 $ & 76.55$ \pm 1.2 $ & 87.30$ \pm 1.1$ & 14.25 & 64.33 \\
SGC  & 33.70$ \pm 3.5 $ & 46.90$ \pm 1.7 $ & 53.49$ \pm 5.1 $ & 58.68$ \pm 3.8 $ & 60.43$ \pm 5.1 $ & 85.11$ \pm 0.5 $ & 76.01$ \pm 1.8 $ & 87.07$ \pm 1.2 $ & 15.50 & 62.67 \\
MixHop  & 60.50$ \pm 2.5 $ & 43.80$ \pm 1.5 $ & 75.88$ \pm 4.9 $ & 73.51$ \pm 6.3 $ & 77.84$ \pm 7.7 $ & 85.31$ \pm 0.6 $ & 76.26$ \pm 1.3 $ & 87.61$ \pm 0.9 $ & 12.38 & 72.59 \\
APPNP  & 47.50$ \pm 1.8 $ & 33.51$ \pm 2.0 $ & 81.42$ \pm 4.3 $ & 77.02$ \pm 7.0 $ & 78.37$ \pm 6.0 $ & 87.94$ \pm 0.6 $ & 77.06$ \pm 1.7 $ & 87.71$ \pm 1.3 $ & 11.75 & 71.32 \\
Geom-GCN  & 60.9 & 38.14 & 64.12 & 60.81 & 67.57 & 90.05 & \textbf{77.99} & 85.27 & 11.38 & 68.11 \\
$\text{H}_{\text{2}}$GCN  & 60.11$ \pm 2.2 $ & 36.48$ \pm 1.9 $ & 87.65$ \pm 5.0 $ & 82.70$ \pm 5.3$ & 84.86$ \pm 7.2 $ & 89.49$ \pm 0.4 $ & 77.11$ \pm 1.6 $ & 87.87$ \pm 1.2 $ & 7.69 & 75.78 \\
FAGCN  & 56.31$ \pm 3.2 $ & 42.43$ \pm 2.1 $ & 81.56$ \pm 4.6 $ & 76.12$ \pm 7.7 $ & 78.11$ \pm 5.0 $ & 85.74$ \pm 0.4 $ & 74.86$ \pm 2.4 $ & 83.21$ \pm 2.0 $ & 13.88 & 72.29 \\
LINKX & 68.42$ \pm 1.4 $ & 61.81$ \pm 1.8 $ & 75.49$ \pm 5.7 $ & 77.84$ \pm 5.8 $ & 74.60$ \pm 8.4 $ & 87.86$ \pm 0.8 $ & 73.19$ \pm 1.0 $ & 84.64$ \pm 1.1 $ & 12 & 75.48 \\
GPRGNN & 46.58$ \pm 1.7 $ & 31.61$ \pm 1.2 $ & 82.94$ \pm 4.2 $ & 80.27$ \pm 8.1 $ & 78.38$ \pm 4.4 $ & 87.54$ \pm 0.4 $ & 77.13$ \pm 1.7 $ & 87.95$ \pm 1.2 $ & 10.94 & 71.55 \\
GGCN & 71.14$ \pm 1.8 $ & 55.17$ \pm 1.6 $ & 86.86$ \pm 3.3 $ & 85.68$ \pm 6.6 $ & 84.86$ \pm 4.6 $ & 89.15$ \pm 0.4 $ & 77.14$ \pm 1.5 $ & 87.95$ \pm 1.1 $ & 5.50 & 79.74 \\
GCNII  & 63.86$ \pm 3.0 $ & 38.47$ \pm 1.6 $ & 80.39$ \pm 3.4 $ & 77.86$ \pm 3.8 $ & 77.57$ \pm 3.8 $ & \textbf{90.15$\pm$0.4} & 77.33$ \pm 1.5 $ & \textbf{88.37$\pm$1.3} & 7.75 & 74.25 \\
DeepGCNs & 48.75$\pm$2.6 & 31.23$\pm$1.4 & 72.75$\pm$4.8 & 68.38$\pm$5.9 &70.27$\pm$7.1 & 88.80$\pm$0.4 &75.58$\pm$1.2 & 85.61$\pm$1.9 & 14.75 & 67.67  \\
$\text{UDGNN}_{\text{GCN}}$ & 74.53$\pm$1.2 & 68.13$\pm$2.6 & 87.64$\pm$3.7 & 84.32$\pm$7.3 & 84.60$\pm$5.3 & 89.85$\pm$0.4 & 76.05$\pm$1.8 & 86.97$\pm$1.2 & 5.75 & 81.51  \\
Conn-NSD & 65.21$\pm$2.0 & 45.19$\pm$1.6 & \textbf{88.73$\pm$4.5} & 85.95$\pm$7.7 & 86.16$\pm$2.2 & 89.28$\pm$0.4 & 75.61$\pm$1.9 & 83.74$\pm$2.2 & 7.31 & 77.48  \\
GloGNN & 69.78$\pm$2.4 & 57.54$\pm$1.4 & 87.06$\pm$3.5 & 83.51$\pm$4.3 & 84.32$\pm$4.2 & 89.62$\pm$0.4 & 77.41$\pm$1.7 & 88.31$\pm$1.1 & 4.63 & 79.69  \\
GloGNN++ & 71.21$\pm$1.8 & 57.88$\pm$1.8 & 88.04$\pm$3.2 & \textbf{85.95$\pm$5.1} & 84.05$\pm$4.9 & 89.24$\pm$0.4 & 77.22$\pm$1.8 & 88.33$\pm$1.1 & 3.81 & 80.24  \\
GPNet (ours)  & \textbf{78.61$\pm$0.2} & \textbf{71.57$\pm$0.1} & 87.45$\pm$0.0 & 84.10$\pm$0.1 & \textbf{87.84$\pm$0.0} & 89.18$\pm$0.0 & 77.20$\pm$0.1 & 88.21$\pm$0.1 & \textbf{3.75} & \textbf{83.02}  \\
\bottomrule
\end{tabular}}
\end{table*}
\begin{table}[!htb]
\centering
\caption{\textbf{The test F1-micro score (\%) over the methods on Reddit}. OOM denotes out of memory.}
\label{tab:4}
\vspace{-0.2cm}
\setlength{\tabcolsep}{4mm}{
\begin{tabular}{ll}
\toprule
Method  & Reddit \\
\midrule
SAGE-GCN & 93.0 \\
SAGE-LSTM & 95.4 \\
SAGE-mean & 95.0 \\
FastGCN & 93.7 \\
GaAN & \textbf{96.4} \\
DGI & 94.0 \\
SGC  & 94.9    \\
GCN  & OOM \\
GAT & OOM    \\
GCNII  & OOM \\
GPRGNN  & OOM \\
GPNet (ours)  & 95.8 $\pm$ 0.0 \\
\bottomrule
\end{tabular}}
\end{table}
\subsection{Model Configurations}
\label{sec:configuration}
We implement our proposed GPNet using Adam \cite{51} optimizer in Pytorch \cite{52}. We now report the search space of hyperparameters that are optimized repeatedly on each dataset. The learning rate is searched from \{0.0003, 0.01, 0.05, 0.1\}, the dropout is searched from \{0, 0.1, 0.3, 0.8, 0.95\}, the $L2$ regularization factor is searched from \{\text{$\text{1e}^{-\text{10}}$}, \text{$\text{1e}^{-\text{7}}$}, \text{$\text{7e}^{-\text{6}}$}, \text{$\text{5e}^{-\text{5}}$}, \text{$\text{6e}^{-\text{5}}$}, \text{$\text{1e}^{-\text{4}}$}, \text{$\text{2e}^{-\text{4}}$}, \text{$\text{5e}^{-\text{4}}$}, \text{$\text{6e}^{-\text{3}}$}\}, the training epochs are searched from {700, 800, 1,000, 1,200, 2,000, 2,200, 5,000, 7,000, 50,000}, the $k$ is searched from \{2, 3, 4, 5, 7, 8, 9, 13\}, the $m$ is searched from \{2, 3\}, the $q_{1}$ is searched from \{2, 4, 5\}, the $q_{2}$ is searched from \{2, 5, 6\}, the $d_{2}$ is searched from \{1, 3\}, the $d_{3}$ is searched from \{6, 9\}, and the aggregation function is searched from {Max-FP, Min-FP, Avg-FP, Sum-FP}. The $d_{1}$, $q_{0}$ and $q_{3}$ are set to 0, 1 and 6, respectively. We set the self-attention score $\alpha$ based on different types of networks, with $\alpha=1$ on citation networks, with $\alpha=0$ on Wikipedia networks, with $\alpha=20, \alpha=19, \alpha=70$ on Cornell, Texas, and Wisconsin, and with $\alpha=0.2$ on Reddit. We set the sign factor $\beta=-1$ on Texas and Wisconsin and $\beta=1$ on other datasets. All results of the proposed methods are averaged over 10 runs with random weight initializations.
\subsection{Baselines}
\label{sec:baseline}
\mysection{Semi-supervised node classification} For the semi-supervised node classification tasks, we compare our method with the following baselines: (1) shallow GNNs: GCN \cite{1}, GAT \cite{2}, SGC \cite{23}, MixHop-learn \cite{3}, APPNP \cite{18}, and ARMA \cite{add_2}; (2) deep GNNs: JKNet \cite{20}, GCNII \cite{22}, and FAGCN \cite{46}. Furthermore, we follow Rong \etal \cite{21} to equip DropEdge \cite{21} on two backbones: IncepGCN \cite{21} and JKNet \cite{20}.

\mysection{Full-supervised node classification} For the full-supervised node classification, we compare GPNet with 18 baselines, including (1) MLP; (2) classical GNNs: GCN \cite{1}, GAT \cite{2}, SGC \cite{23}, MixHop \cite{3}, and APPNP \cite{18}; (3) GNNs for heterophily settings: (4) Geom-GCN \cite{45}, $\text{H}_{\text{2}}$GCN \cite{53}, FAGCN \cite{46}, LINKX \cite{54}, GPRGNN \cite{48}, and GGCN \cite{55}; (5) deep GNNs: GCNII \cite{22} and DeepGCNs \cite{56}; (6) recent SOTA GNNs: UDGNNGCN \cite{57}, Conn-NSD \cite{58}, GloGNN \cite{59}, and  GloGNN++ \cite{59}.

\mysection{Inductive learning} For the large graph Reddit, several general models that target large graphs are listed as baselines: SAGE-GCN \cite{32}, SAGE-LSTM \cite{32}, SAGE-mean \cite{32}, FastGCN \cite{9}, GaAN \cite{39}, and DGI \cite{61}. In addition, on the large graph, we also compare our method with recent SOTA GNNs, including GCN \cite{1}, GAT \cite{2}, SGC \cite{23}, GCNII \cite{22}, and GPRGNN \cite{48}.
\begin{figure*}
    \centering
    \includegraphics[width=0.9\textwidth]{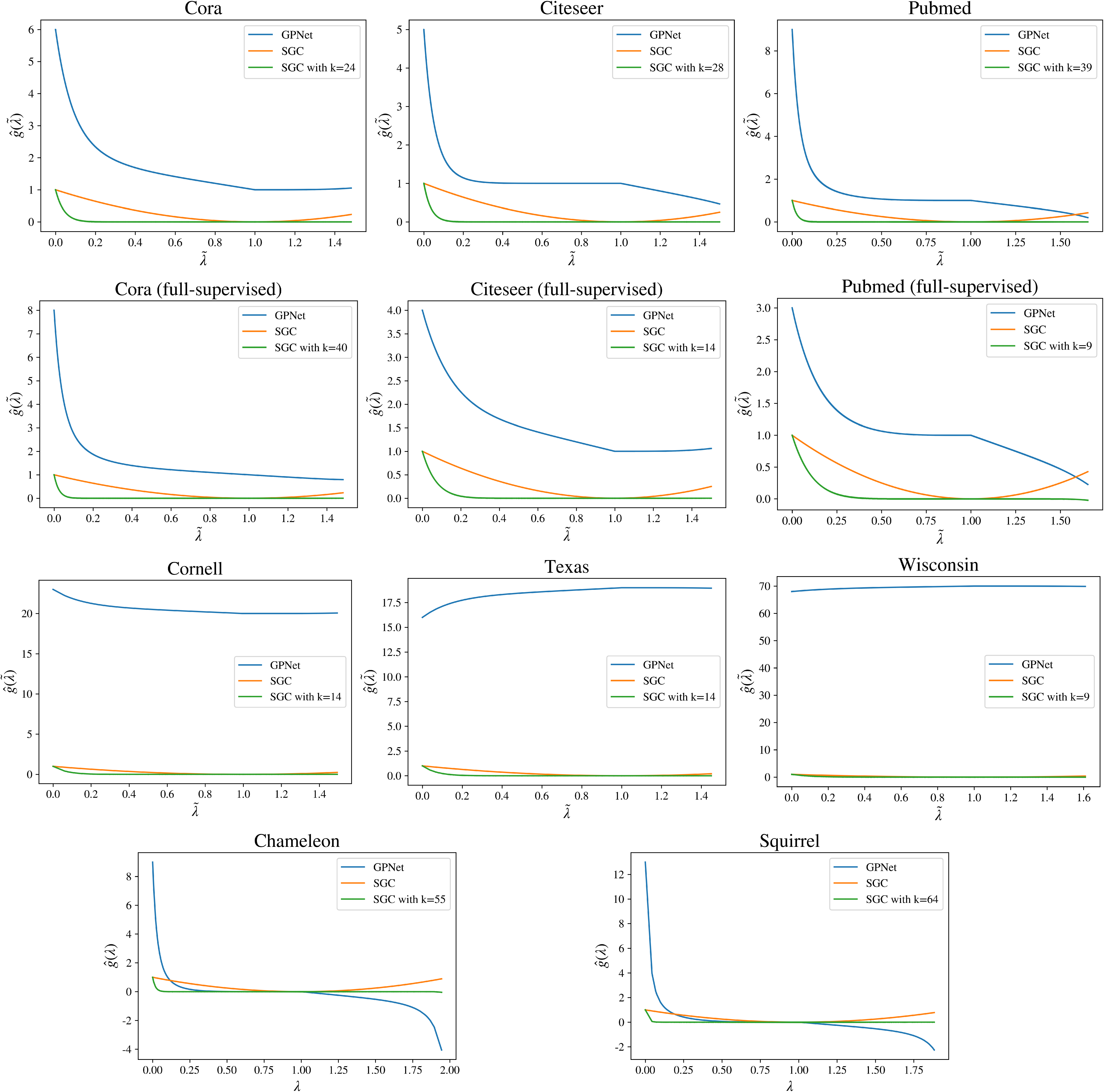}
    \caption{\textbf{Filters learnt on graph datasets by GPNet and SGC}. $\lambda / \tilde{\lambda}$ are the eigenvalues with / without self-loops, and $\hat{g}(\tilde{\lambda}) / \hat{g}(\lambda)$ are the corresponding filter coefficients.
    }
\label{fig:3}
\end{figure*}
\begin{figure}[!hb]
    \centering
    \includegraphics[page=9, width=0.6\columnwidth]{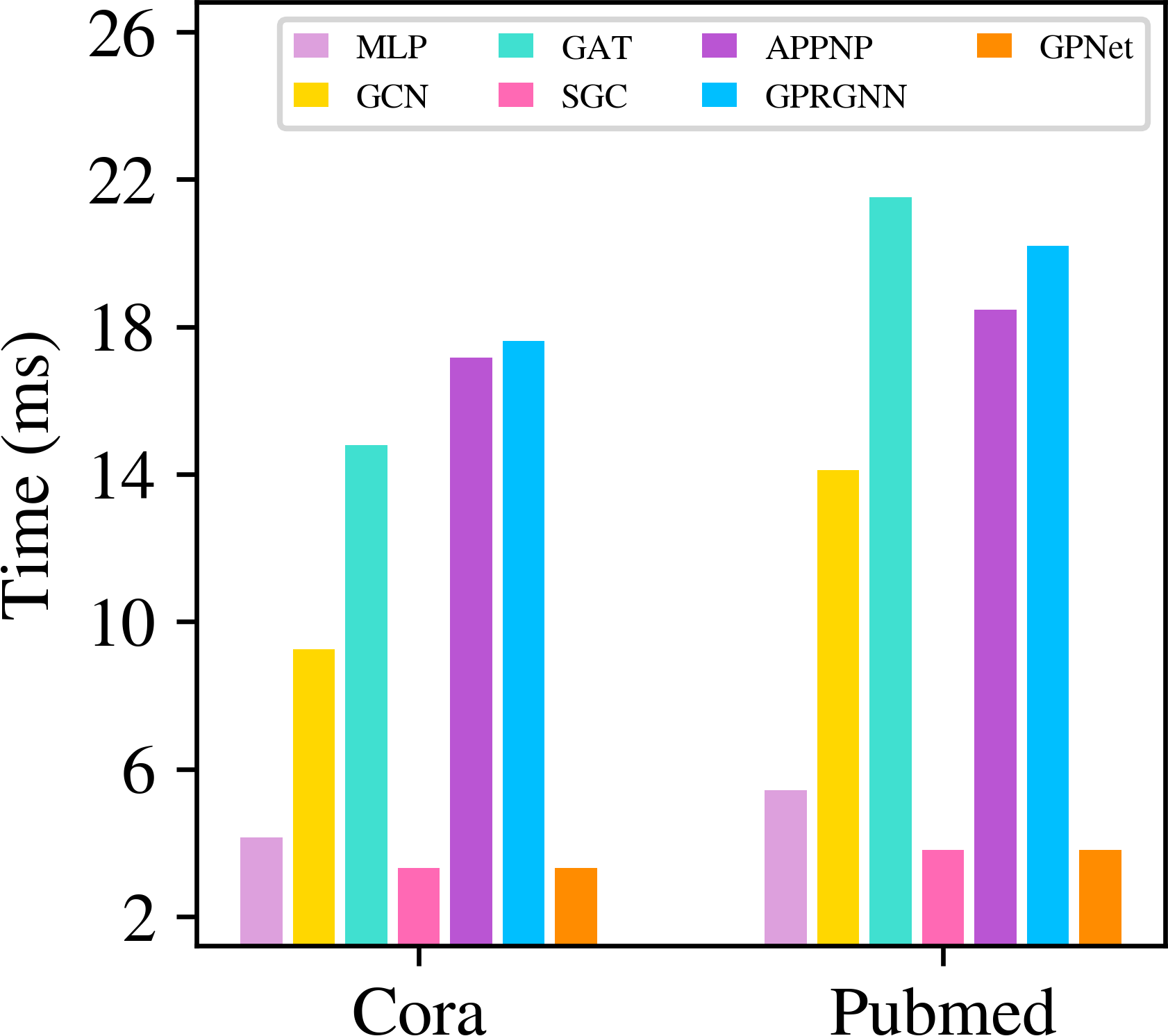}
    \caption{\textbf{Performance over average training time per epoch on Cora and Pubmed}.}
\label{fig:4}
\end{figure}
\begin{figure*}[!htb]
    \centering
    \includegraphics[width=\textwidth]{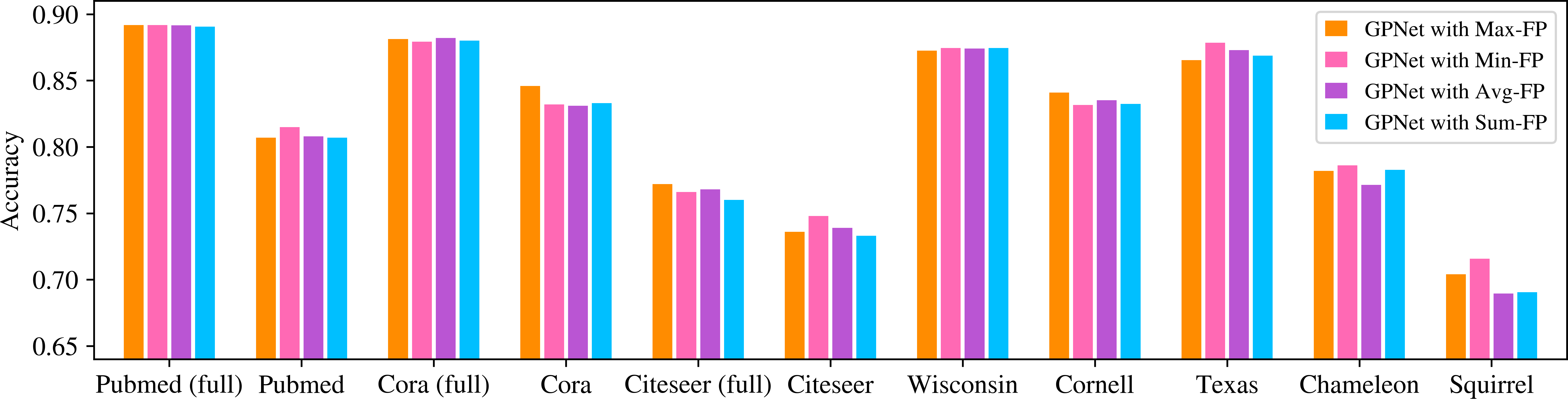}
    \caption{\textbf{Performance of GPNet with different aggregation functions on both homophily and heterophily datasets}.}
\label{fig:5}
\end{figure*}
\begin{figure*}[!htb]
    \centering
    \includegraphics[width=\textwidth]{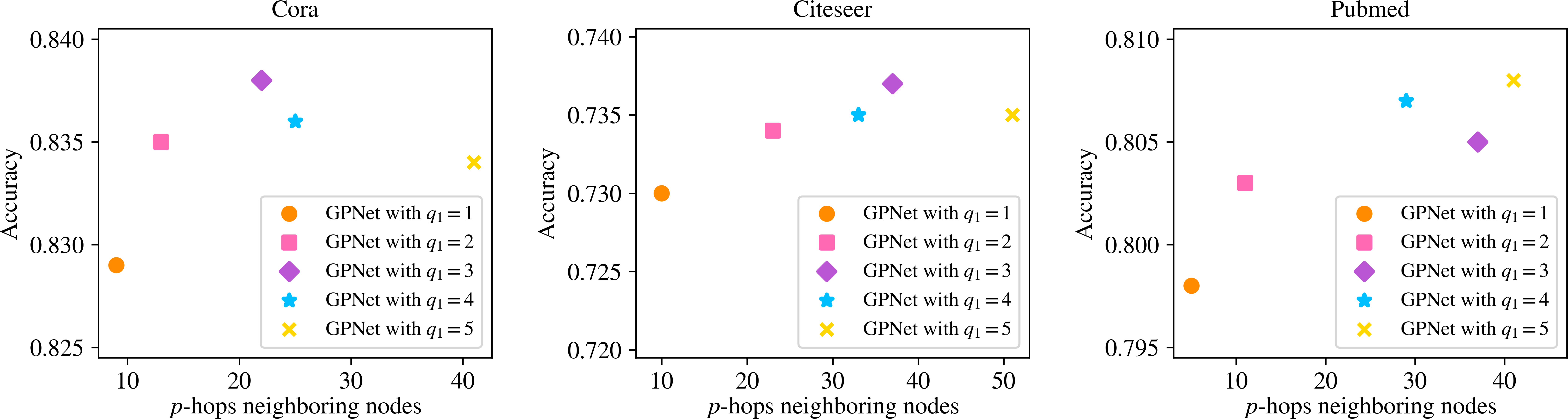}
    \caption{\textbf{Effect of GPNet with various dilation factors ($q_{1}$)}. $x$-axis denotes the p-hops neighboring nodes captured by the respective models to achieve the best performance.}
\label{fig:6}
\end{figure*}

\begin{figure}[!htb]
    \centering
    \includegraphics[width=0.7\columnwidth]{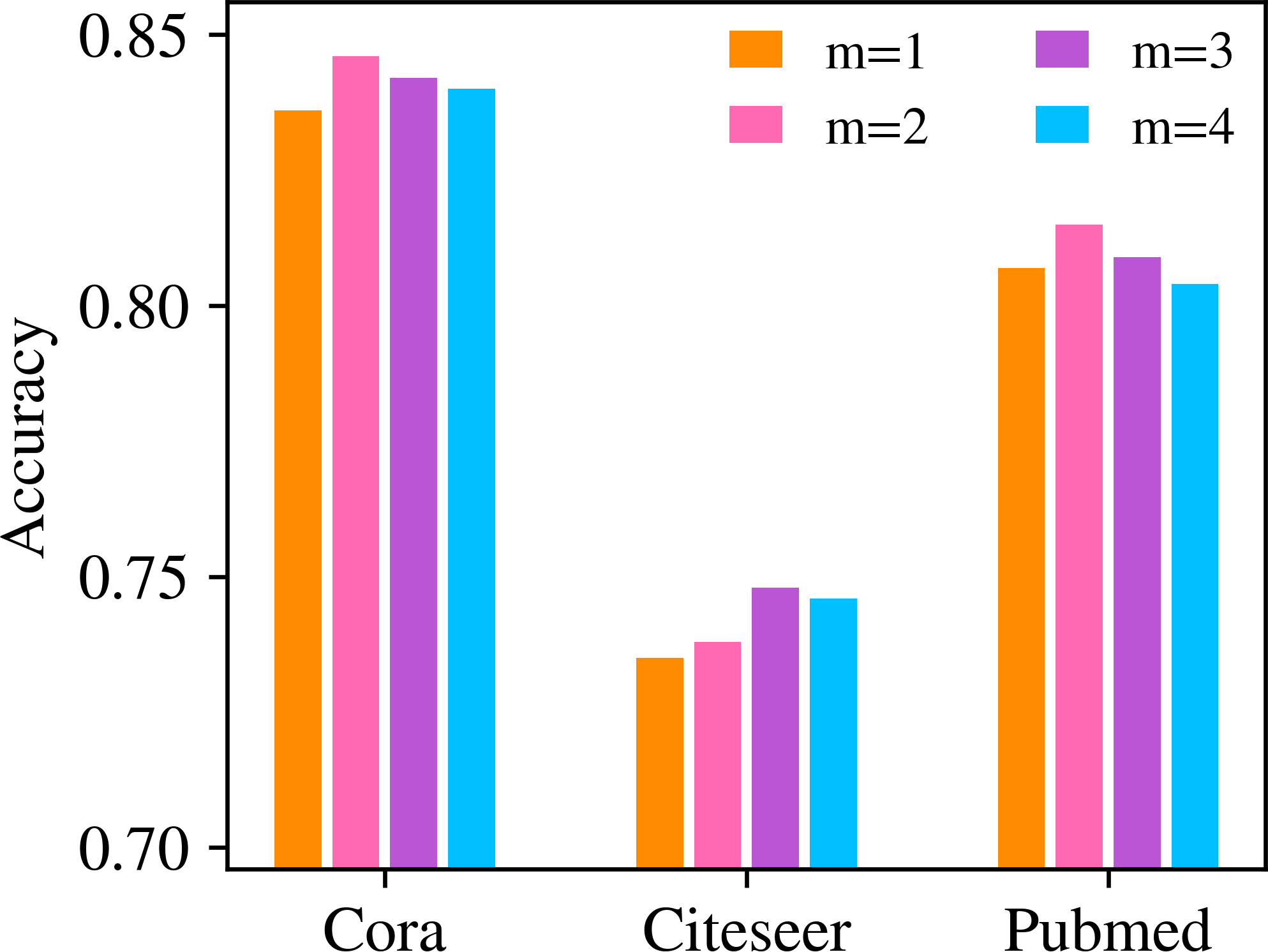}
    \caption{\textbf{Effect of channels ($m$)}.}
\label{fig:7}
\end{figure}

\subsection{Results}
\label{sec:result}
\mysection{Results for semi-supervised node classification} \tblLabel \ref{tab:2} summarizes the performance of GPNet to these SOTA baselines on the semi-supervised node classification tasks. We have these findings: (1) SGC that removes the activation function achieves similar performance to GCN model, which suggests that non-linear functions are not necessary to improve the results. (2) High-order GNNs such as MixHop-learn and APPNP improve the expressive power of GCN with one-order. (3) Deep GNNs such as GCNII outperform GCN by significant margins. This shows that these techniques of deep GNNs can effectively address the over-smoothing of multi-layer GCN. (4) Instead of stacking network layers, GPNet, obtains the SOTA performance in terms of average rank and average accuracy among all baselines on Pubmed, Citeseer, and Cora, using the geometric polynomials of adjacency matrices with various dilation factors, improves over GCN by 3.2\%, 6.4\%, and 3.8\%, improves over SGC by 3.3\%, 4.0\%, and 4.4\%, and improves over FAGCN by 2.6\%, 2.9\%, and 0.6\%, respectively. These results provide positive answers to question Q1.

\mysection{Results for full-supervised node classification} To support answering Q2 and Q4, we compare GPNet with these SOTA node classification baselines on the full-supervised node classification tasks, and the results are summarized in \tblLabel \ref{tab:3}. From the table, we make the following observations: (1) MLP that uses only node's own feature performs surprisingly well on webpage networks (Wisconsin, Cornell, and Texas), but achieves poor performance on these networks with a medium-scale number of nodes and edges (Wikipedia networks and citation networks). This shows the importance of node's own feature to small-scale graphs. (2) GCNII with the initial residual and identity mapping significantly outperforms that of GCN, GAT, and SGC, and achieves best performance with homophily graphs. This shows that the initial residual and identity mapping mechanisms effectively address the over-smoothing of classical GNN models. (3) Overall, GNNs with heterophily settings achieve superior results on heterophily datasets. For example, on Chameleon, Squirrel, Wisconsin, Cornell, and Texas datasets, the accuracy scores of GGCN are 71.14\%, 55.17\%, 86.86\%, 85.68\%, and 84.86\%, respectively. These results provide strong support for the techniques used on heterophily datasets. (4) We observe that GPNet achieves new SOTA performance in terms of average rank and average accuracy over all the datasets. This demonstrates the effectiveness of the proposed model to performance improvement on both homophily and heterophily datasets.

\mysection{Results for inductive learning} \tblLabel \ref{tab:4} shows the results of GPNet and competing methods on Reddit. On Reddit, we observe that GPNet ranks second and the GNNs such as GAT fail to run. These results provide positive answers to question Q3.
\begin{figure*}[!htb]
    \centering
    \includegraphics[width=\textwidth]{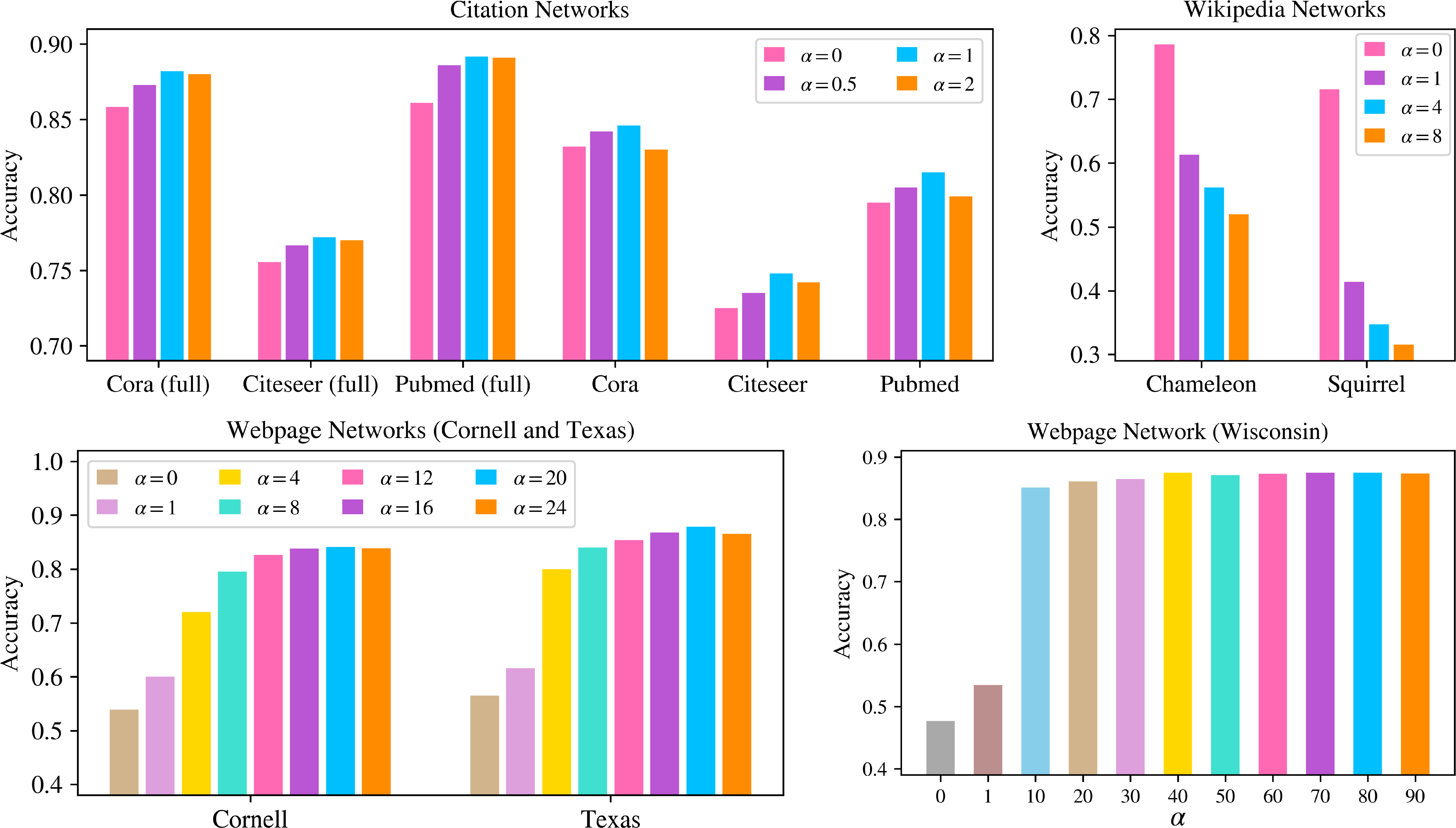}
    \caption{\textbf{Effect of self-attention score ($\alpha$)}. Citation networks are homophily networks, and other networks are heterophily networks.}
\label{fig:8}
\end{figure*}
\begin{figure*}[!htb]
    \centering
    \includegraphics[width=\textwidth]{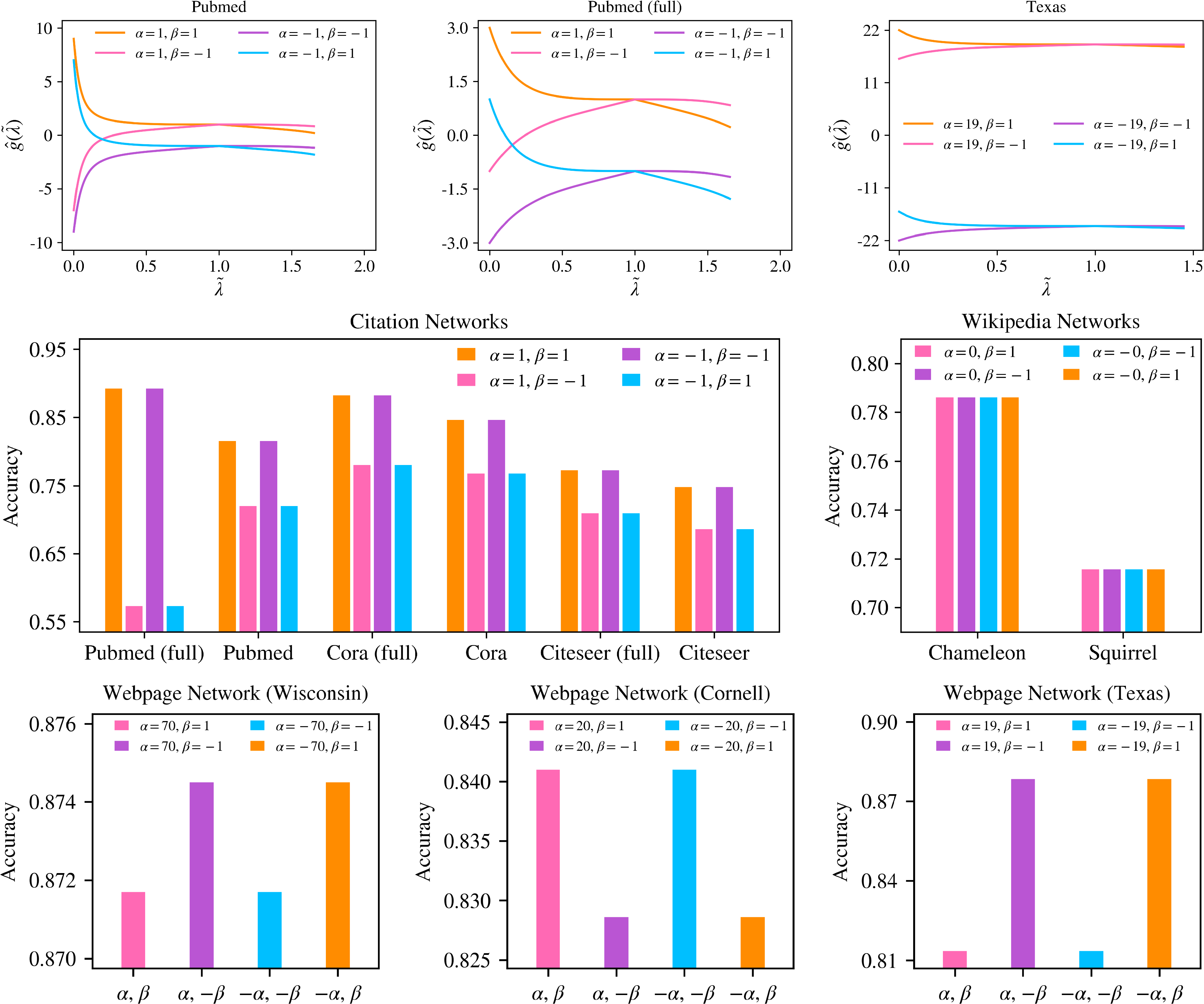}
    \caption{\textbf{Effect of sign factor ($\beta$) and filter}. Citation networks are homophily networks, and other networks are heterophily networks.}
\label{fig:9}
\end{figure*}

\subsection{Learning of Filters}
\label{sec:efficiency}
\figLabel \ref{fig:3} depicts the filters learnt by GPNet and SGC on different datasets. On all datasets, SGC retains the low-frequency signals of node features while filtering out the high-frequency signals of node features. However, the high-frequency signals are helpful for heterophily graphs \cite{46}. Therefore, SGC is optimal for homophily datasets but fails to work well on heterophily datasets. \figLabel \ref{fig:3} shows that SGC with $k=24$ model suffers from the over-smoothing and almost ignores all node features on Cora. GPNet pays attention to signals of different frequencies according to different graph statistics (homophily ratio, nodes, and edges). In GPNet, we mainly consider the low-frequency signals while paying less attention to the high frequency signals on homophily datasets. In addition, on Cornell, Texas, and Wisconsin datasets with heterophily and small-scale nodes and edges, GPNet captures both the low-frequency and high-frequency signals. On other heterophily datasets, GPNet mainly considers the high-frequency signals.

\subsection{Analysis of Efficiency and Parameters}
\label{sec:efficiency}
\figLabel \ref{fig:4} lists the performance of MLP and the representative GNNs over their training time per epoch relative to that of GPNet on Cora and Pubmed datasets. In \figLabel \ref{fig:4}, in terms of training time per epoch, GPNet achieves as good performance as SGC, and takes fewer computations compared to other baselines. The larger datasets see the most benefit. For example, on Pubmed, compared with GCN, GPNet achieves a 73.0\% speed up. As shown in \secLabel \ref{sec:time}, we can precompute the parameter-free feature extraction $\bar{H}$ in the same way as SGC. Then GPNet and SGC only learn a single weight matrix during training. This indicates that both of them can consistently provide fewer computations and parameters, which provides positive answers to question Q5.
\subsection{Ablation Study}
\label{sec:ablation}
\mysection{Effect of Aggregation function} We conduct the experiments using the proposed aggregation functions on both homophily and heterophily datasets. The results are summarized in \figLabel \ref{fig:5}. The best performance on different datasets varies slightly for GPNet models with different aggregation functions. Specifically, GPNet with Max-FP achieves the highest test accuracy on Cora, Citeseer (full), and Cornell datasets, GPNet with Min-FP achieves the best results on Pubmed (full), Pubmed, Citeseer, Texas, Chameleon, and Squirrel datasets, GPNet with Avg-FP obtains the highest classification accuracy on Cora (full) dataset, and GPNet with Sum-FP obtains the best performance on Wisconsin dataset.

\mysection{Effect of dilation} We take GPNet with $m=1$ as an example, repeatedly optimize the number of terms of the geometric polynomial using different dilation factors to achieve the best test accuracy for the respective models. The results are summarized in \figLabel \ref{fig:6}. We observe that GPNet models with $q_{1}>1$ can capture the features of neighboring nodes at more distant distances and achieve better performance, compared to GPNet with $q_{1}=1$.

\mysection{Effect of channel} \figLabel \ref{fig:7} compares the performance of GPNet models with different channels. From the results, we observe that GPNet models with $m$=2 and 3 outperform GPNet with $m$=1 by a large margin. This shows that multiple channels contribute to the performance improvement.

\mysection{Effect of self-attention score} \figLabel \ref{fig:8} shows the comparison results of the proposed models with various self-attention scores for different types of networks. From the figure, we make the following observations: (1) On the citation networks, the accuracy of GPNet improves as the self-attention score increases (increases $\alpha$) until $\alpha$ of 1. By further increasing $\alpha$, the accuracy of GPNet begins to decrease. (2) On the medium-scale Wikipedia networks, GPNet with $\alpha=0$ achieves the best performance, and the performance of the model decreases as $\alpha$ increases. (3) On the small-scale webpage networks, GPNet with $\alpha=20$ outperforms other methods on Cornell and Texas, and GPNet with $\alpha=40$ and GPNet with $\alpha=70$ are better than other methods. These observations imply that self-nodes and neighboring nodes that are $p$-hops away play an equally important role in training homophily networks, self-nodes are harmful for performance improvement on medium-scale heterophily networks, and self-nodes are particularly important in training small-scale heterophily networks compared to the neighboring nodes.

\mysection{Effect of sign factor and filter} \figLabel \ref{fig:9} summarizes the performance of various sign factors. We observe that different sign factors help performance in training different types of networks. Notably, we find that their models achieve the same performance when the graph filters are symmetric to each other about the $x$-axis. This shows that the performance depends on the non-zero part of the spectral coefficients, regardless of the filter type. This theoretical explanation can be found in \secLabel \ref{sec:sgc}.

\mysection{Effect of terms} \figLabel \ref{fig:10} depicts the validation results achieved by varying the number of terms $k\in${1..9} of different channels. We see that the overall accuracy of the proposed models with different channels tends to improve as $k$ increases. Compared with GPNet with $m$=1, GPNet with $m$=2 attains higher performance across various terms. The results further support the design of multi-channel learning.
\begin{figure*}[!htb]
    \centering
    \includegraphics[width=\textwidth]{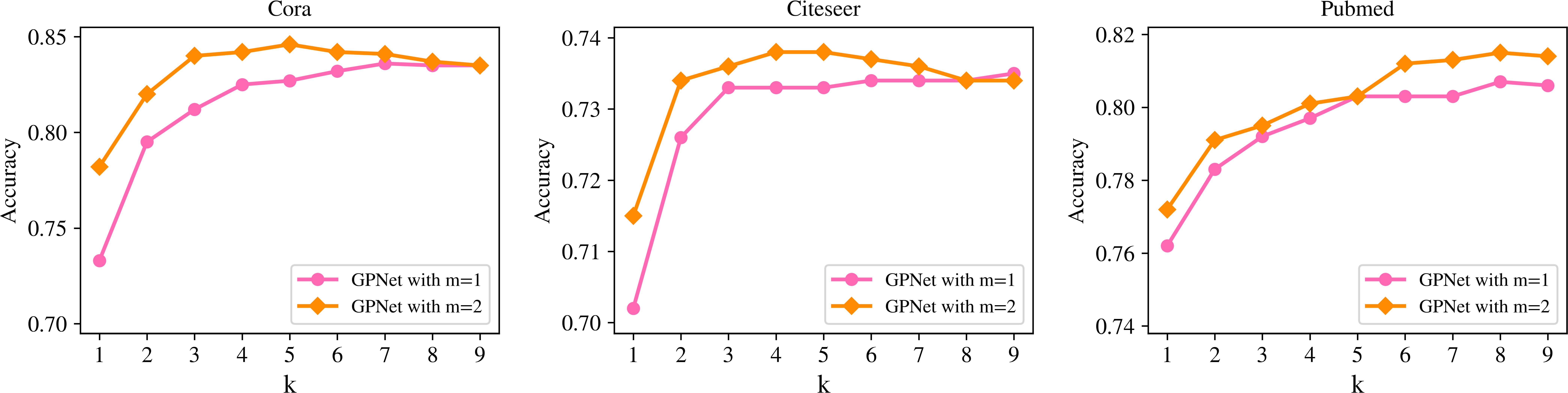}
    \caption{\textbf{Effect of terms ($k$)}.}
\label{fig:10}
\end{figure*}

\begin{table*}[!htbp]
	\renewcommand{\thetable}{6}
	\centering
	\caption{\textbf{Effect of GPNet with / without ReLU}. GPNet-ReLU is GPNet with ReLU.}
	\label{tab:6}
	\vspace{-0.2cm}
	\setlength{\tabcolsep}{1mm}{
		\begin{tabular}{l|lllll|lll|lll}
			\toprule
			Method  & Chameleon & Squirrel & Wisconsin	& Cornell & Texas & Pubmed (full) & Citeseer (full) & Cora (full) & Pubmed & Citeseer & Cora \\
			\midrule
			GPNet-ReLU & 60.56$\pm$4.1 & 31.44$\pm$5.0 & 83.41$\pm$3.1 & 82.04$\pm$1.4 & 86.41$\pm$0.5 & 86.33$\pm$1.8 & 75.56$\pm$1.2 & 78.53$\pm$2.1 & 72.2$\pm$11.0 & 74.3$\pm$0.1 & 84.4$\pm$0.1  \\
			GPNet (ours)  & \textbf{78.61$\pm$0.2} & \textbf{71.57$\pm$0.1} & \textbf{87.45$\pm$0.0} & \textbf{84.10$\pm$0.1} & \textbf{87.84$\pm$0.0} & \textbf{89.18$\pm$0.0} & \textbf{77.20$\pm$0.1} & \textbf{88.21$\pm$0.1} & \textbf{81.5$\pm$0.0} & \textbf{74.8$\pm$0.1} & \textbf{84.6$\pm$0.1}  \\
			\bottomrule
	\end{tabular}}
\end{table*}
\begin{table}[!htbp]
	\renewcommand{\thetable}{5}
	\centering
	\caption{\textbf{Effect of adjacency matrix with / without self-loops}. The average degree is equal to 2 times the number of edges divided by the number of nodes, and A. M. and SL are Adjacency Matrix and Self-loops, respectively.}
	\label{tab:5}
	\vspace{-0.2cm}
	\setlength{\tabcolsep}{2mm}{
		\begin{tabular}{l|c|cc}
			\toprule
			Dataset & Average Degree & A. M. with SL & A. M. without SL \\
			\midrule
			Cora & 4.01 & 84.6 $\pm$ 0.1 & 83.7 $\pm$ 0.4 \\
			Citeseer & 2.84 & 74.8 $\pm$ 0.1 & 73.2 $\pm$ 0.2 \\
			Pubmed & 4.50 & 81.5 $\pm$ 0.0 & 80.0 $\pm$ 0.0 \\
			Cora (full) & 4.01 & 88.21 $\pm$ 0.1 & 87.80 $\pm$ 0.0 \\
			Citeseer (full) & 2.84 & 77.20 $\pm$ 0.1 & 76.46 $\pm$ 0.1 \\
			Pubmed (full) & 4.50 & 89.18 $\pm$ 0.0 & 88.92 $\pm$ 0.0 \\
			\midrule
			Chamelean & 31.71 & 71.03 $\pm$ 0.1 & 78.61 $\pm$ 0.2    \\
			Squirrel & 76.14 & 63.33 $\pm$ 0.1 & 71.57 $\pm$ 0.1 \\
			Wisconsin & 3.98 & 87.45 $\pm$ 0.0 & 87.10 $\pm$ 0.1    \\
			Cornell & 3.22 & 84.10 $\pm$ 0.1 & 83.35 $\pm$ 0.1 \\
			Texas & 3.38 & 87.84 $\pm$ 0.0 & 84.86 $\pm$ 0.0 \\
			\bottomrule
	\end{tabular}}
\end{table}

\mysection{Effect of adjacency matrix}
We investigate the effect of whether adjacency matrix has self-loops. As shown in \tblLabel \ref{tab:5}, GPNet with $\hat{A}$ (adjacency matrix with self-loops) consistently works better on these datasets with small average degree while performing slightly worse on other datasets compared to GPNet without $\hat{A}$. This study demonstrates that adjacency matrix with self-loops is suitable for graphs with small average degree, while adjacency matrix without self-loops contributes to the performance of graphs with large average degree.

\mysection{Effect of ReLU} \tblLabel \ref{tab:6} shows the effect of whether GPNet has ReLU on all datasets. We see that GPNet without ReLU attains higher accuracy and more stable performance compared to GPNet with ReLU. This may be because the activation function destroys the captured feature information. Therefore, we propose to remove the nonlinear activation function in the design of GNNs.

\section{Conclusion}\label{sec:conclusion}
In this paper, to address the limitations of existing GNNs, we have investigated the proposed designs that improve the expressive power of learning graph representations. With the designs, we design the GPNet model that learns a distribution under optimal efficiency and parameters by precomputing the fixed feature extraction on both heterophily and homophily networks. We theoretically analyze that the model can correspond to different types of filters. Experiments on various graph learning tasks show that the model achieves significant performance gains over the state-of-the-art methods. For future work, an interesting direction is to apply the proposed model to directed networks.

% if have a single appendix:
%\appendix[Proof of the Zonklar Equations]
% or
%\appendix  % for no appendix heading
% do not use \section anymore after \appendix, only \section*
% is possibly needed

% use appendices with more than one appendix
% then use \section to start each appendix
% you must declare a \section before using any
% \subsection or using \label (\appendices by itself
% starts a section numbered zero.)
%

%\input{sections/supplement.tex}
% \section{Proof of the First Zonklar Equation}
% Appendix one text goes here.

% % you can choose not to have a title for an appendix
% % if you want by leaving the argument blank
% \section{}
% Appendix two text goes here.

% use section* for acknowledgment
\ifCLASSOPTIONcompsoc
  % The Computer Society usually uses the plural form
  \section*{Acknowledgments}
\else
  % regular IEEE prefers the singular form
  \section*{Acknowledgment}
\fi

This research was funded by the National Natural Science Foundation of China (Grant No. 42274016/D0402, U1701266), the Program for Guangdong Introducing Innovative and Entrepreneurial Teams (Grant No.2019ZT08L213), the Natural Science Foundation of Guangdong Province (Grant No.2021A1515011483, 2022A1515011146), the Guangdong Provincial Key Laboratory of Intellectual Property and Big Data (Grant No. 2018B030322016), and Special Projects for Key Fields in Higher Education of Guangdong, China (Grant No. 2020ZDZX3077,2021ZDZX1042).

% Can use something like this to put references on a page
% by themselves when using endfloat and the captionsoff option.
\ifCLASSOPTIONcaptionsoff
  \newpage
\fi

% trigger a \newpage just before the given reference
% number - used to balance the columns on the last page
% adjust value as needed - may need to be readjusted if
% the document is modified later
%\IEEEtriggeratref{8}
% The "triggered" command can be changed if desired:
%\IEEEtriggercmd{\enlargethispage{-5in}}

% references section

% can use a bibliography generated by BibTeX as a .bbl file
% BibTeX documentation can be easily obtained at:
% http://mirror.ctan.org/biblio/bibtex/contrib/doc/
% The IEEEtran BibTeX style support page is at:
% http://www.michaelshell.org/tex/ieeetran/bibtex/
%\bibliographystyle{IEEEtran}
% argument is your BibTeX string definitions and bibliography database(s)
%\bibliography{IEEEabrv,../bib/paper}
%
% <OR> manually copy in the resultant .bbl file
% set second argument of \begin to the number of references
% (used to reserve space for the reference number labels box)
%\begin{thebibliography}{1}
% \newpage
\bibliographystyle{IEEEtran}
\bibliography{references}
%\end{thebibliography}
% biography section
%
% If you have an EPS/PDF photo (graphicx package needed) extra braces are
% needed around the contents of the optional argument to biography to prevent
% the LaTeX parser from getting confused when it sees the complicated
% \includegraphics command within an optional argument. (You could create
% your own custom macro containing the \includegraphics command to make things
% simpler here.)
%\begin{IEEEbiography}[{\includegraphics[width=1in,height=1.25in,clip,keepaspectratio]{mshell}}]{Michael Shell}
% or if you just want to reserve a space for a photo:
% \newpage
% \input{sections/supplement.tex}

% \newpage
% You can push biographies down or up by placing
% a \vfill before or after them. The appropriate
% use of \vfill depends on what kind of text is
% on the last page and whether or not the columns
% are being equalized.

%\vfill

% Can be used to pull up biographies so that the bottom of the last one
% is flush with the other column.
%\enlargethispage{-5in}
%\input{sections/supplement.tex}

% that's all folks
\end{document}